\documentclass[10pt,journal,compsoc]{IEEEtran}
\usepackage{lastpage} 
\usepackage{amssymb}
\usepackage{color}
\usepackage{amsmath}
\usepackage{amsthm}
\usepackage{amsmath}
\usepackage{amsfonts}
\usepackage{url}
\usepackage{multirow}
\usepackage{footnote}
\usepackage{amssymb}
\usepackage{mathtools}
\usepackage{algorithmicx}
\usepackage{algorithm}
\usepackage{algpseudocode} 
\usepackage{cases}
\usepackage{cite}
\usepackage{makecell}
\usepackage{tabularx}
\usepackage[table]{colortbl}
\usepackage{booktabs}
\usepackage{tikz}
\usetikzlibrary{arrows,shapes,fit}
\usepackage{caption}
\usepackage{subcaption}
\usepackage{dblfloatfix} 
\usepackage{flushend}
\usepackage{standalone}
\usepackage{soul,xcolor}
\usepackage{color, colortbl}
\usepackage{rotating}
\usepackage{blindtext}
\usepackage{pgfplots}
\pgfplotsset{compat=newest}
\usepackage[thinlines]{easytable}
\usepackage{enumitem}
\usepackage{ragged2e}

\input{mysymbol.sty}
\usepackage{needspace}

\input{pennColors.sty}

\newcommand{\herm}{\mathsf{H}}

\definecolor{Gray}{gray}{0.75}

\setstcolor{black}

\ifCLASSINFOpdf
\else
\fi

\begin{document}

\title{Graph-Time Convolutional Neural Networks: \\ Architecture and Theoretical Analysis}

\author{Mohammad Sabbaqi and Elvin Isufi
\IEEEcompsocitemizethanks{\IEEEcompsocthanksitem The authors are with the Intelligent Systems Department, Delft University of Technology, Delft, The Netherlands. 
\protect\\
E-mails: \{m.sabbaqi, e.isufi-1\}@tudelft.nl;
\protect\\
Part of this work has been presented in~\cite{isufi2021graphproduct}.
}
\thanks{This work is supported by the TU Delft AI Labs programme.}}

%
%


\IEEEtitleabstractindextext{%
\begin{abstract}
\justifying{
Devising and analysing learning models for spatiotemporal network data is of importance for tasks including forecasting, anomaly detection, and multi-agent coordination, among others.
Graph Convolutional Neural Networks (GCNNs) are an established approach to learn from time-invariant network data.
The graph convolution operation offers a principled approach to aggregate multi-resolution information in each layer and offers some degree of mathematical analysis by exploring tools from graph signal processing.
This analysis provides insights on the equivariance properties of GCNNs; spectral behaviour of the learned filters; and the stability to perturbations in the graph topology, which arises because of support perturbations or uncertainties.
However, extending the convolution-principled learning and respective analysis to the spatiotemporal domain is challenging because spatiotemporal data have more intrinsic dependencies.
Hence, a higher flexibility to capture \emph{jointly} the spatial and the temporal dependencies is required to learn meaningful higher-order representations.
Here, we leverage product graphs to represent the spatiotemporal dependencies in the data and introduce Graph-Time Convolutional Neural Networks (GTCNNs) as a principled architecture to aid learning.
The proposed approach can work with any type of product graph and we also introduce a parametric product graph to learn also the spatiotemporal coupling.
The convolution principle further allows a similar mathematical tractability as for GCNNs.
In particular, the stability result shows GTCNNs are stable to spatial perturbations but there is an implicit trade-off between discriminability and robustness; i.e., the more complex the model, the less stable.
Extensive numerical results on benchmark datasets corroborate our findings and show the GTCNN compares favourably with state-of-the-art solutions.
We anticipate the GTCNN to be a starting point for more sophisticated models that achieve good performance but are also fundamentally grounded.}
\end{abstract}
%
}

\maketitle

\IEEEdisplaynontitleabstractindextext

\IEEEpeerreviewmaketitle

\IEEEraisesectionheading{\section{Introduction}}
\label{sec_intro}


Learning from \emph{multivariate temporal} data is a challenging task due to their intrinsic spatiotemporal dependencies.
This problem arises in applications such as time series forecasting, classification, action recognition, and anomaly detection~\cite{mo:yu2017spatio,mo:yan2018STGCNN,mo:kadous2002temporal,mo:zhang2019anomaly}.
The spatial dependencies can be captured by a graph either explicitly such as in transportation networks or implicitly such as in recommender systems~\cite{wang2021graph}.
Therefore, graph-based inductive biases should be considered during learning to exploit the spatial dependencies alongside with temporal patterns in a computationally and data efficient manner.
Based on advances in processing and learning over graphs~\cite{GSPsurvey,hamilton2017representation}, a handful of approaches have been proposed to learn from multivariate temporal data~\cite{surveyDLonST}.
The main challenge is to capture the spatiotemporal dependencies by built-in effective biases in a principled manner~\cite{battaglia2018inductivebias}.

The convolution principle has been key to build learning solutions for graph-based data \cite{gama2020elvinmagazine,bronstein2021geometric}. By cascading graph convolutional filters and pointwise nonlinearities, graph convolutional neural networks (GCNNs) have been developed as non-linear models for graph-based data~\cite{GamaGCNN,CNNonGraphs}. Such a principle reduces both the number of learnable parameters and the computational complexity in the GCNN, ultimately, overcoming the curse of dimensionality~\cite{battaglia2018inductivebias,bronstein2021geometric}. The convolution operation allows also for a mathematical tractability of GCNNs. On the one hand, it is a permutation equivariant operation, implying that it can capture structural symmetries to aid learning~\cite{Ruiz2021GNN,bronstein2021geometric}. On the other hand, it enjoys a spectral duality by means of a graph signal processing analysis \cite{GSPsurvey}, which, in turn, can be used to characterize the stability of GCNNs to perturbations in the support \cite{gamamain,gao2021stability,Ruiz2021GNN,levie2021transferability}. The overreaching goal of this paper is to leverage the convolution principle to learn from spatiotemporal data as well as study the theoretical properties of this model.


\subsection{Related Works}\label{subsec:relW}

There are two stream of works related to this paper: one developing architectures; and one studying their stability.

\smallskip
\noindent\textbf{Architectures.} Spatiotemporal graph-based learning models can be divided into hybrid and fused models.
\emph{Hybrid} models combine distinct learning algorithms for graph and temporal data.
They use GCNNs to extract higher-level spatial features and process these features by a temporal RNN, CNN, or variants of them.
The works in~\cite{mo:chai2018bike,mo:manessi2020dynamic,mo:sun2020constructing} concatenate a GCNN by an LSTM where the former is applied per timestamp and the latter per node.
Instead, the work in~\cite{mo:khodayar2018spatio} uses an RNN followed by a GCNN.
Another work in~\cite{mo:yu2017spatio} uses a narrower graph convolutional layer in between two temporal gated convolution layers, whereas~\cite{wang2021behavioral} combines graph diffusion with temporal convolutions. 
Another approach is to ignore the graph structure and use temporal CNNs augmented by the attention mechanism to reduce the number of parameters and ease training~\cite{mo:guo2019attention}.
Graph WaveNet uses a gated dilated temporal convolution followed by a first order graph convolution~\cite{mo:wu2019graph}.
%
%
%
%
\emph{Fused} models impose the graph structure into conventional spatiotemporal models to jointly capture the spatiotemporal relationships~\cite{mo:isufi2019VARMA,mo:seo2018GRNN,mo:ruiz2020gatedGRNN,mo:si2019AEGCLSTM,mo:yan2018STGCNN}.
The fully connected blocks are replaced by graph convolutions to generate graph-based latent embeddings.
The work in~\cite{mo:isufi2019VARMA} proposed graph-based VARMA to learn spatiotemporal embeddings.
The authors in~\cite{mo:seo2018GRNN,mo:si2019AEGCLSTM} employ variants of RNN models with graph convolution-based blocks, while \cite{mo:ruiz2020gatedGRNN} also designs a graph-based gating module.
In another approach, the spatial graph is passed through a continuous temporal shift operator to define a spatiotemporal convolution~\cite{hadou2021space}. 
The work in~\cite{mo:yan2018STGCNN} starts with a graph partitioning and defines a spatiotemporal neighborhood in each partition to perform message passing.
The works in~\cite{pareja2020evolvegcn,ehsan2019variational} use diffusion schemes where the output is a function of the signals of at most one hop neighbors. Since diffusion is considered in a causal fashion the information from the k-hop spatiotemporal neighbors is considered in subsequent layers, which may be limiting to extract patterns.


The advantage of hybrid models lies within their simple and efficient implementation since they benefit from modular spatial and temporal blocks.
However, this modularity makes it unclear how to best combine them for the task at hand and analyze theoretically their inner-working mechanisms.
Moreover, they are disjoint and sequential in nature (first graph and then temporal processing or vice-versa); thus, are more restrictive and require each node to accumulate global information.
Instead, fused models overcome these issues but they mostly use a low-order spatiotemporal aggregation which limits their effectiveness in capturing spatiotemporal patterns.
We address this challenge by leveraging product graphs to represent multivariate time series~\cite{sandryhaila2014big}.
Product graphs have been widely used for modeling complex data~\cite{leskovec2010kronecker}, the matrices~\cite{huang2018rating}, and developing a graph-time signal processing framework \cite{grassi2017time}.
However, their use for learning from spatiotemporal data has been limited.
The work in~\cite{pan2020spatio} builds spatiotemporal scattering transforms using product graphs and shows their advantages over alternative data representative solutions.
Here, we use product graphs as a platform to run spatiotemporal convolutions and build neural network solutions with it.
The work in \cite{mo:yan2018STGCNN} implicitly uses the Cartesian product graph to build a spatiotemporal neighborhood.
Differently, we develop a principled convolutional framework for learning via any product graph, multi-resolution neighborhood information aggregation in each layer (~\cite{mo:yan2018STGCNN} focuses only on immediate neighbors), and introduce a parametric product graph to learn the spatiotemporal coupling.


\smallskip
\noindent\textbf{Stability.}
Studying the stability of graph neural networks to graph perturbation reveals their potentials and limitations w.r.t. the underlying support.
The work in~\cite{gamamain} provided a stability scheme to relative perturbations on the graph for GCNNs by enforcing graph convolutional filters to vary smoothly over high graph frequencies.
Authors in~\cite{levie2019transferability} proposed a linear stability bound to graph arbitrary perturbation for a large class of graph filters named Cayley smooth space.
The work in~\cite{thanou2021stability} considers an edge rewiring model for perturbation and provides an upper bound for output changes.
The authors in~\cite{gao2021stability} proved that GCNNs are stable to stochastic perturbation model where links have a probability of existence in the graph.
In case of spatiotemporal graph neural networks, the work in~\cite{mo:ruiz2020gatedGRNN} developed a stability analysis to graph perturbation following the model in~\cite{gamamain}.
The authors in~\cite{hadou2021space} have used the relative perturbation model in~\cite{gamamain} to investigate their model stability properties. 
We also pursue a similar analysis as~\cite{gamamain} for the graph-time convolutional neural network to approve its capabilities and observe the effects of time component in the stability and transferability  properties of them.
The similar base of stability analysis allows us to compare the proposed model with previous works in~\cite{mo:ruiz2020gatedGRNN,hadou2021space} and collect insights into how different learning solutions for spatiotemporal data handle uncertainty in the spatial support.
Differently from the latter, the proposed stability results via product graphs provides direct links with that of GCNNs if time is invariant, highlighting also the impact of how the temporal component is accounted for a joint spatiotemporal learning.

\subsection{Contribution}\label{subsec:contrib}

We consider the convolutional principle over product graphs as an inductive bias to build graph-time convolutional neural network (GTCNN) for learning from multivariate temporal data. Our specific contributions are:

\begin{enumerate}[label=C\arabic*)]
	\item \emph{Principled architecture:} We develop an architecture that leverages graph-time convolutions and product graphs as a spatiotemporal inductive bias to reduce the computational cost and the number of trainable parameters.
	We also propose a solution to learn the spatiotemporal coupling and a recursive implementation to tackle the high dimensionality of the product graphs.
	\item \emph{Theoretical properties:}
	We prove GTCNNs are permutation equivariant. Then, leveraging concepts from graph-time signal processing, we show the learned filters act in the joint graph and temporal spectral domain as point-wise multiplication between the learnt frequency response and the graph-time Fourier transform of the input.
	%
	\item \emph{Stability:}
	We prove the GTCNN is stable to perturbations in the spatial graph. Our result shows the GTCNN becomes less stable for larger graphs and that the temporal window in the product graph affects it more severely than increasing the number of nodes. We provide a thorough discussion of our stability bound w.r.t. baselines and alternative approaches.
	\item  \emph{Numerical performance:} We compare the GTCNN with baseline and state-of-the-art solutions in different tasks about time series classification and forecasting. The GTCNN outperforms the baseline GCNN model that ignores the temporal dimension and compares well with alternatives on benchmark datasets.
\end{enumerate}

The structure of this paper is as follows.
In Section.~\ref{sec_back}, we explain product graphs and signals over them to formulate our problem.
In Section.~\ref{sec_gtcnn}, the GTCNN is proposed and its properties are discussed.
Section.~\ref{sec_stability} is dedicated to stability analysis of GTCNNs.
Section.~\ref{sec_numeric} contains the numerical results and experiments to corroborate the strengths of GTCNNs.
Finally, Section~\ref{sec_conc} concludes the paper.


\section{Problem Formulation}
\label{sec_back}
Here, we first introduce some background material about product graphs and their use to represent multivariate time-series. Then, we motivate the problem of learning from multivariate time-series via an inductive bias perspective.

\subsection{Product Graphs for Spatiotemporal Coupling}\label{subsec_PG}
%

Consider an $N \times 1$ multivariate signal $\bbx_t$ collected over $T$ time instances in matrix $\bbX = [\bbx_1, \ldots, \bbx_T ]$. The $t$th column $\bbx_t = [x_t(1), \ldots, x_t(N)]^\top$ may be measurements of different sensors in a sensor network at time $t$, whereas the $i$th row $\bbx^i = [x_1(i), \ldots, x_T(i)]^\top$ may be the time series of sensor $i$ over the $T$ time instants. When learning representations from the data in $\bbX$ we are interested in exploiting both the spatial dependencies in $\bbx_t$ and the temporal dependencies in $\bbx^i$. However, despite learning from these spatial and temporal relations is important to extract patterns in $\bbX$, we need to devise methods that exploit them \emph{jointly} as an inductive bias for learning representations~\cite{battaglia2018inductivebias}.

\begin{figure}[t]
	\begin{center}
		\includegraphics[width=0.7\linewidth]{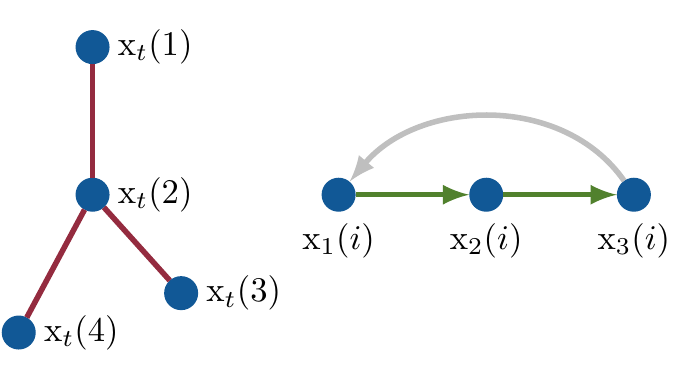}
	\end{center}
	\caption{Spatial and temporal graphs along with signals over their nodes. (Left) Spatial graph and graph signal at time $t$. Scalar $x_t(2)$ is the signal and node two recorded at time $t$. It is proximal with signals at nodes one, three, and four. (Right) Temporal graph and $i$th time series illustrated as graph signal. Edges in green are those of a directed line graph, while edges green and grey are of a cyclic graph.}
	\label{fig: signal over graphs}	
\end{figure}

When data $\bbx_t$ have a (hidden) underlying structure, we can capture their \textit{spatial} relations through a \textit{spatial graph} $\ccalG = (\ccalV, \ccalE)$ of $N$ nodes in set $\ccalV = \{1, \ldots , N\}$ and $|\ccalE|$ edges in set $\ccalE \subseteq \ccalV \times \ccalV$. We represent the structure of $\ccalG$ via its graph shift operator (GSO) matrix $\bbS \in  \reals^{N \times N}$ which is a sparse matrix with non-zero entries $[\bbS]_{ij} \neq 0$ only if $(i,j) \in \ccalE$ or $i = j$~\cite{ortega2018overview}. We can view $\bbx_t$ as a graph signal with scalar $x_t(i)$ residing on node $i$. Likewise, we can capture the \textit{temporal} relations in $\bbx^i$ through a \textit{temporal graph} $\ccalG_T = (\ccalV_T , \ccalE_T )$ of $T$ nodes in set $\ccalV_T = \{1, \ldots, T\}$
and $|\ccalE_T |$ edges in set $\ccalE_T \in \ccalV_T \times \ccalV_T$. Each node represents one time instant $t$ and set $\ccalE_T$ contains an edge if signals at instances $t$ and $t^\prime$ are related. Also, we represent the structure of $\ccalG_T$ with its GSO matrix $\bbS_T \in \reals^{T \times T}$. Data $\bbx^i$ can be seen as a graph signal with scalar $x_t(i)$ being the value at node $t$. Examples of $\ccalG_T$ are the directed line graph that assumes signal $x_t(i)$ depends only on the former instance $x_{t-1}(i)$, the cyclic graph that accounts for periodicity, or any other graph that encodes the temporal dependencies in $\bbx^i$, see Fig.~\ref{fig: signal over graphs}.

Given graphs $\ccalG$, $\ccalG_T$, we can capture the spatiotemporal relations in the data $\bbX$ through product graphs
\begin{equation}
\mathcal{G}_{\diamond}=   \mathcal{G}_{T}   \diamond \mathcal{G} =\left(\mathcal{V}_{\diamond}, \mathcal{E}_{\diamond} ,\mathbf{S}_{\diamond}\right).
\end{equation}
The node set $\ccalV_\diamond= \ccalV_T \times \ccalV$ is the Cartesian product between $\ccalV_T$ and $\ccalV$ with cardinality $|\mathcal{V}_{\diamond}|=NT$ and node $i_\diamond \in \ccalV$ represents the space-time location $(i,t)$; see Fig.~\ref{fig: product graph}. The edge set $\ccalE_\diamond \subseteq \mathcal{V}_\diamond \times \mathcal{V}_\diamond$ connects now space-time locations, which structure and respective GSO $\bbS_\diamond \in \reals ^{NT \times NT}$ are dictated by the type of product graph. Typical product graphs are~\cite{sandryhaila2014big}:
\begin{itemize}
	\item \textit{Kronecker product}: $\ccalG_{\otimes} = \ccalG_T {\otimes} \ccalG = (\ccalV_{\otimes}, \ccalE_\otimes, \bbS_\otimes)$ has the GSO $\bbS_{\otimes}=\bbS_T \otimes \bbS$. The number of edges is $|\ccalE_\otimes| = 2|\ccalE| |\ccalE_T|$. The Kronecker product translates the spatial coupling into a temporal coupling in which spatiotemporal nodes $i_\diamond = (i,t)$ and $j_\diamond = (j,t+1)$ are connected only if the spatial nodes $i,j$ are connected in $\ccalG$ and the temporal nodes $t$ and $t+1$ are connected in $\ccalG_T$; grey edges in Fig.~\ref{fig: product graph}.
	\medskip
	\item \textit{Cartesian product}: $\ccalG_{\times} = \ccalG_T {\times} \ccalG = (\ccalV_{\times}, \ccalE_\times, \bbS_\times)$ has the GSO $\mathbf{S}_{ \times}=\mathbf{S}_T \otimes \mathbf{I}_{N}+\mathbf{I}_{T} \otimes \mathbf{S}$. The number of edges is $|\ccalE_\times| = T|\ccalE| + N|\ccalE_T|$. The Cartesian product implies a temporal coupling by connecting each spatial node to itself in consecutive instants as dictated by $\ccalG_T$, i.e., the spatiotemporal nodes $i_\diamond = (i,t)$ and $j_\diamond = (i,t+1)$ are connected if time instants $t$ and $t+1$ are adjacent; red and green edges in Fig.~\ref{fig: product graph}.

	\medskip
	\item \textit{Strong product}: $\ccalG_{\boxtimes} = \ccalG_T {\boxtimes} \ccalG = (\ccalV_{\boxtimes}, \ccalE_\boxtimes, \bbS_\boxtimes)$ has the GSO $\bbS_{\boxtimes}=\mathbf{S}_T \otimes \mathbf{I}_{N}+\mathbf{I}_{T} \otimes \mathbf{S} + \bbS_T \otimes \bbS$. The number of edges is $|\ccalE_\boxtimes| = |\ccalE| |\ccalE_T| + T|\ccalE| + N|\ccalE_T|$. The strong product is the union of the Kronecker and Cartesian products. Here, spatiotemporal nodes $i_\diamond = (i,t)$ and $j_\diamond=(j,t+1)$ are connected if time instants $t$ and $t+1$ are adjacent in the temporal graph $\ccalG_T$ and spatial nodes $i$ and $j$ are either neighbors or $i=j$ in graph $\ccalG$; red, green, and grey edges in Fig.~\ref{fig: product graph}.
	
\end{itemize}

\begin{figure}[!t]
	\begin{center}
		\includegraphics[width=\linewidth]{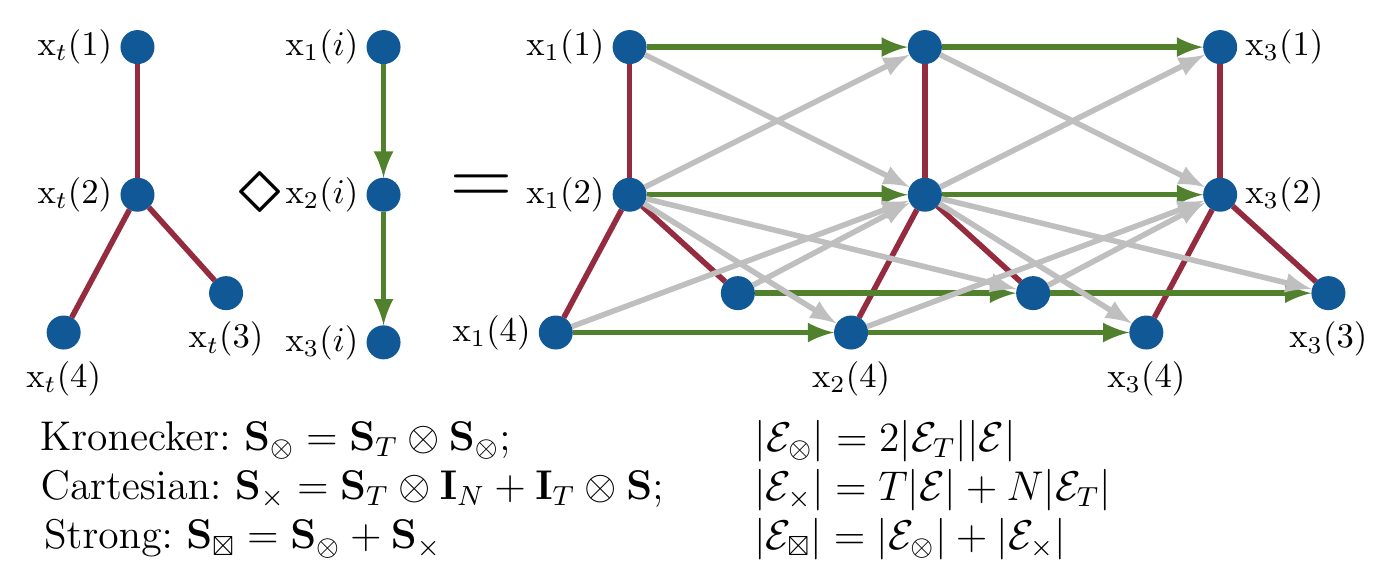}
	\end{center}
	\caption{Different product graphs and the formation of multivariate temporal data over them. (grey edges) Kronecker product $\ccalG_{\otimes} = \ccalG_T {\otimes} \ccalG$. (green and red edges) Cartesian product $\ccalG_{\times} = \ccalG_T {\times} \ccalG$. (All edges) Strong product $\ccalG_{\boxtimes} = \ccalG_T {\boxtimes} \ccalG.$}
	\label{fig: product graph}
\end{figure}

%
Column-vectorizing $\bbX$ into $\bbx_\diamond =\operatorname{vec}(\bbX) \in \reals^{NT}$ we obtain a product graph signal in which the $i$th entry is the signal value at the spatiotemporal node $i_\diamond$. Some of such values are illustrated in Fig.~\ref{fig: product graph}.

\subsection{Problem Motivation}\label{subsec_probMotiv}

We are interested in learning representations from the spatiotemporal data in $\bbX$ for a particular task such as inferring the source of a signal~\cite{Zhao2020anomalydet}, forecasting future values~\cite{mo:yu2017spatio}, or classifying time series~\cite{mo:yan2018STGCNN}. While the product graph is a viable mathematical tool to represent the structure in $\bbX$, we can also ignore it and process $\bbX$ with standard approaches such as the vector autoregressive (VAR) models~\cite{lutkepohl2005new}, multilayer perceptrons (MLPs)~\cite{friedman2017elements}, or recurrent neural networks (RNNs)~\cite{rumelhart1986learning}. However, such techniques face two main limitations. First, their number of parameters and computational complexity is of order $\ccalO(N^2)$ as they involve dense layers, suffering the curse of dimensionality~\cite{bronstein2021geometric}. Second, they depend on the fixed order of the time series. The latter is a challenge because the learned representations do not permute with the time series and, when a new time series becomes available, the whole system needs to be retrained. Consequently, these models are not transferable.

One strategy to overcome such challenges is to exploit separately the spatial and the temporal structure by first processing $\bbX$ column-wise with a graph neural network (GNN)~\cite{GNNsurvey} and successively with a temporal convolution or RNN~\cite{mo:chai2018bike,mo:khodayar2018spatio}. However, such strategy fails to exploit the spatiotemporal coupling in the data and leverages only the coupling in the learned higher-level representation of the GNN, ultimately, limiting the performance and interpretability. One way to overcome this challenge is to embed a graph structure within a recursive model such as VAR~\cite{mo:isufi2019VARMA} or RNN~\cite{mo:seo2018GRNN,mo:ruiz2020gatedGRNN}. Working with graph-based VAR models is limiting because they work in the linear space, while graph-based RNN works only with a fixed latent-space dimension of $N$ (number of the nodes) or require an additional pooling step. Both the latter are uncomfortable because RNN-type of algorithms work best with a lower-dimensional latent space and finding the appropriate pooling strategy adds an extra challenge to the whole system~\cite{GNNsurvey}. Another aspect of RNNs is that they suffer training issues in their vanilla form; hence, LSTMs~\cite{hochreiter1997long} or GRUs~\cite{cho2014learning} are needed, which put more emphasis on the temporal dependencies rather than on the joint spatiotemporal ones.

Differently, we propose a first-principle neural network solution that exploits product graphs as an inductive bias for the spatiotemproal coupling, that is modular to any pooling technique, and that puts equal emphasis on the joint spatiotemporal learning. Building on first principles allows for a mathematical tractability akin to that of the principled spatial CNN~\cite{gama2020elvinmagazine}. To achieve these goals, we leverage the shift-and-sum principle of the convolution operation~\cite{oppenheim1999discrete} to propagate information over the product graph but differently from a GCNN~\cite{gama2020elvinmagazine} we leverage the sparsity of the product graph to handle its large dimensionality. Working with the convolution principle allows also developing an architecture that is equivariant to permutations~\cite{gama2020elvinmagazine}, enjoys a spectral analysis via the graph-time Fourier transform~\cite{grassi2017time,elvin2dfilters} (Section~\ref{sec_gtcnn}), and that is stable to perturbations in the spatial support (Section~\ref{sec_stability}).
\section{Graph-Time Convolutional Neural Networks}
\label{sec_gtcnn}

In this section, we define the graph-time convolutional neural network (GTCNN) and discuss its properties both in the vertex and in the spectral domain.


\subsection{Graph-Time Convolutional Filters}\label{subsec_gtcf}

\begin{figure*}[!t]
	\centering
	\input{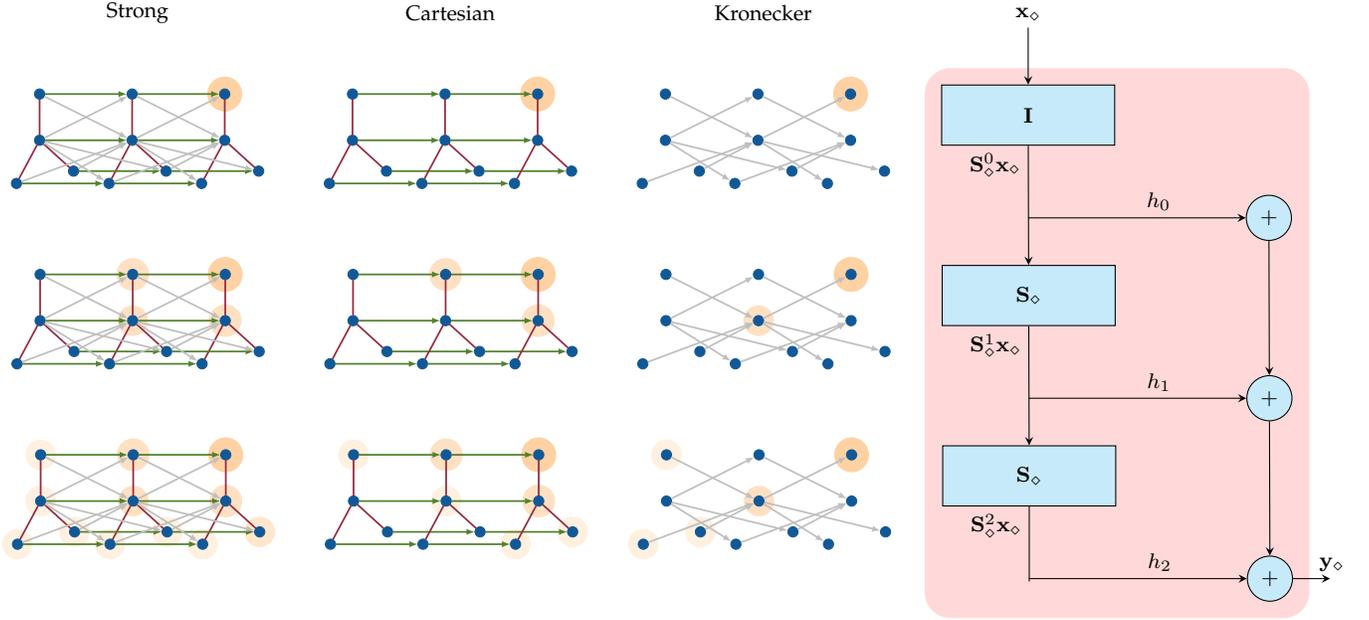}
	\caption{Graph time convolutional filter of order two illustrated for a Cartesian product graph.
		Each block shifts the signal $\bbx_\diamond$ over the product graph by its graph shift operator $\bbS_\diamond$. Each shift implies a neighborhood coverage highlighted in yellow for a particular node. Each shifted signal $\bbS_\diamond^k \bbx_{\diamond}$ is scaled by its filter coefficient $h_k$ and summed up to build the output $\bby_{\diamond}$. Changing the product graph implies a different spatiotemporal neighborhood coverage.}
	\label{fig:gt_filter}
\end{figure*}

The key element to build a GTCNN is the \emph{graph-time convolutional filter}. For a product graph $\ccalG_\diamond = (\ccalV_\diamond,\ccalE_\diamond,\bbS_\diamond)$ and signal $\bbx_\diamond$, the output $\bby_\diamond$ of a graph-time convolutional filter of order $K$ is
\begin{equation}\label{eq:graphtimefilter}
\bby_\diamond = \bbH(\bbS_\diamond)\bbx_\diamond = \sum_{k=0}^K h_k \bbS_\diamond^k \bbx_\diamond
\end{equation}
where $h_0,\dots,h_K$ are the parameters and $ \bbH(\bbS_\diamond):= \sum_{k=0}^K h_k \bbS_\diamond^k$ denotes the filtering matrix. The qualifier convolution for filter \eqref{eq:graphtimefilter} comes from the fact that it shifts the input $\bbx_\diamond$ up to $K$ times over the product graph $\bbS_\diamond^0\bbx_\diamond, \bbS_\diamond^1\bbx_\diamond, \ldots, \bbS_\diamond^K\bbx_\diamond$ and builds the output $\bby_\diamond$ as a scaled sum of these shifts; see Fig.~\ref{fig:gt_filter}. Since $\bbS_\diamond$ is spatiotemporally local, the shift $\bbS_\diamond\bbx_\diamond$ diffuses the signal from the spatiotemporal node $i_\diamond$ to any other immediate spatiotemporal neighbor $j_\diamond$. From recursion $\bbS_\diamond^k\bbx_\diamond = \bbS_\diamond(\bbS_\diamond^{k-1}\bbx_\diamond)$, we can see that higher-order powers $\bbS_\diamond^k$ diffuse the signal to spatiotemporal neighbors that up to $k-$hops. This implies that filter \eqref{eq:graphtimefilter} is spatiotemporally local in a neighborhood of radius $K$. The filter resolution depends on the order $K$ but also on the type of product graph. For instance, the Kronecker product ignores the information present at time $t$ limiting the spatial resolution but has a high temporal one; the Cartesian product brings in the spatial information of time $t$ but accounts for the temporal proximity only for signal values at the same node; the strong product accounts for both and needs a lower order $K$ to cover a particular neighborhood; Fig.~\ref{fig:gt_filter}. 

The locality of $\bbS_\diamond$ allows obtaining the output $\bby_\diamond$ in~\eqref{eq:graphtimefilter} with a computational cost of order $\ccalO(K|\ccalE_\diamond|)$. The latter can be achieved by leveraging the recursion $\bbS_\diamond^k\bbx_\diamond= \bbS_\diamond(\bbS_\diamond^{k-1}\bbx_\diamond)$. This complexity order is appealing despite we work with a large GSO and input vector. However, it should be noted that $|\ccalE_\diamond|$ is in turn governed by the type of product graph [cf. Fig.~\ref{fig: product graph}], thus making the strong product a suitable choice only if the spatial graph is highly sparse. It also follows from~\eqref{eq:graphtimefilter} that the order of parameters is $\ccalO(K)$ and it is independent on the spatial and temporal dimensions.

\subsection{GTCNNs}\label{subsec_gtcnn}

A graph-time convolutional neural network (GTCNN) is a composition of graph-time convolutional filters with pointwise nonlinearities. Specifically, consider an architecture composed of $L$ layers $\ell = 1, \ldots, L$ and let $\bbH_\ell(\bbS_\diamond) = \sum_{k=0}^K h_{k\ell} \bbS_\diamond^k$ be the filter used at layer $\ell$.
The GTCNN propagation rule is
\begin{equation}
\bbx_{\diamond,\ell} = \sigma \big( \bbH_\ell(\bbS_\diamond)\bbx_{\diamond,\ell-1}		\big) = \left(\sum_{k=0}^K h_{k\ell} \bbS_\diamond^k\bbx_{\diamond,\ell-1}		\right)
\end{equation}
where $\bbx_{\diamond,0}\coloneqq\bbx_\diamond$ is the GTCNN input. Fig.~\ref{fig:GTCNN_bd} illustrates a GTCNN of three layers. To augment the representation power, we consider multiple nodal features per layer in matrix $\bbX_{\diamond,\ell-1} = [\bbx_{\diamond,\ell-1}^1, \ldots, \bbx_{\diamond,\ell-1}^F] \in \reals^{NT \times F}$ where each column $\bbx_{\diamond,\ell-1}^g$ is a graph-time signal feature at layer $\ell-1$. These features are passed through a bank of graph-time convolutional filters and pointwise nonlinearities to yield the output features of layer $\ell$
\begin{equation}\label{eq.filter_bankMat}
\bbX_{\diamond,\ell} = \sigma\left(\sum_{k=0}^K \bbS_\diamond^k\bbX_{\diamond,\ell-1}\bbH_{k\ell}	\right)
\end{equation}
where $\bbH_{k\ell} \in \reals ^{F \times F}$ is the parameter matrix containing the coefficients of the filter bank at layer $\ell$. To ease the analysis of the filer bank in \eqref{eq.filter_bankMat}, we make explicit the input-output relation between the $f$th output feature $\bbx_{\diamond,\ell}^f$ and the $g$th input feature $\bbx_{\diamond,\ell-1}^g$ as
\begin{equation}\label{eq.eq.filter_bank}
\bbx_{\diamond,\ell}^f = \sigma \bigg(\sum_{g = 1}^F \bbH_\ell^{fg}(\bbS_\diamond)\bbx_{\diamond,\ell-1}^g	\bigg) = \sigma \bigg(	\sum_{g = 1}^F	\sum_{k=0}^K h_{k\ell}^{fg}\bbS_\diamond^k\bbx_{\diamond,\ell-1}^g	\bigg)
\end{equation}
for all $f = 1, \ldots, F$ and where $\bbH_\ell^{fg}(\bbS_\diamond) = \sum_{k=0}^K h_{k\ell}^{fg}\bbS_\diamond^k$ is one of the $F^2$ filters used in layer $\ell$.

Recursions \eqref{eq.filter_bankMat} (resp. \eqref{eq.eq.filter_bank}) are repeated for all layers. In the last layer $\ell = L$, we consider without loss of generality the number of output features is one, i.e., $\bbx_{\diamond,L} := \bbx_{\diamond,L}^1$. This output is a function of the input $\bbx_\diamond$ and the collection of all filter banks $\bbH_\ell^{fg}(\bbS_\diamond)$ [cf. \eqref{eq:graphtimefilter}]. Grouping all filters in the filter tensor $\ccalH(\bbS_\diamond) = \{	\bbH_\ell^{fg}(\bbS_\diamond)	\}_{\ell fg}$, we can define the GTCNN output as the mapping
\begin{equation}\label{eq:GTCNN}
\bbx_{\diamond,L}:= \Phi(\bbx_\diamond;\ccalH(\bbS_\diamond))\quad \text{with}\quad \ccalH(\bbS_\diamond) = \{	\bbH_\ell^{fg}(\bbS_\diamond)	\}_{\ell fg}.
\end{equation}
%
The GTCNN parameters are learned to minimize a loss $\ccalL(\bbx_L, \bbY)$ computed over a training set of input-output pairs $\ccalT = \{\bbx_\diamond, \bbY\}$.

As it follows from \eqref{eq:graphtimefilter}, the number of parameters in each filter is $K+1$. This gets scaled by the number of filters per layer $F^2$ and the number of layers $L$, yielding an order of $\ccalO(KF^2L)$ parameters defining the GTCNN. Likewise, the computational complexity for running the filter bank for $L$ layers is of order $\ccalO(KF^2L|\ccalE_\diamond|)$. Such orders are similar to those of conventional GCNNs working with time invariant signals. This is a consequence of exploiting the sparsity of product graphs and the spatiotemporal coupling via graph convolutional filters [cf. \eqref{eq:graphtimefilter}]. In the next section, we shall discuss how to \emph{learn} this spatiotemporal coupling.


\smallskip
\noindent
\textbf{Alternative graph-time neural networks.}
	Once the product graph is built, we can employ any aggregation scheme to collect spatiotemporal information. For the sake of completness, we detail here the message passing neural network (MPNN)~\cite{hamilton2017representation} and the graph attention network (GAT)~\cite{velivckovic2017graph}.

	\emph{MPNNs} update a spatiotemporal node's latent representation $\bbx_{\diamond,\ell}$ by passing a message from its neighborhood as
	\begin{equation}
		\bbm_{\diamond,\ell-1}^{(i)} = f(\bbx_{\diamond,\ell-1}[j]: (i,j)\in \ccalE_\diamond)
	\end{equation}
	where $\bbm_{\diamond,\ell-1}^{(i)}$ is the message vector for node $i$, $\bbx_{\diamond,\ell-1}[j]$ is the output of layer $\ell - 1$ at node $j$, and $f(\cdot)$ is a differentiable and permutation equivariant function.
	Each node's latent vector can be updated based on inferred messages and its current feature vector as
	\begin{equation}
		\bbx_{\diamond,\ell} = g(\bbx_{\diamond,\ell-1},\bbM_{\diamond,\ell-1})
	\end{equation}
	where $\bbM_{\diamond,\ell-1} = \{\bbm_{\diamond,\ell-1}^{(i)}\}$ collects all the messages for each node and $g(\cdot)$ is an arbitrary differentiable function (i.e., neural networks).
	Since we are considering the spatiotemporal product graph, the messages are including spatiotemporal information, and the type of product graph dictates which nodes are forwarding these messages.
	 
	\emph{GATs} use an attention mechanism to update latent representation of each node based on its neighborhood as
	\begin{equation}
		\bbx_{\diamond,l+1}[i] = \sigma\left(\sum_{(i,j)\in\ccalE_\diamond} \alpha_{ij}\bbH\bbx_{\diamond,\ell}[j]\right)
	\end{equation}
	where $\alpha_{ij}$ indicates attention coefficients and $\bbH \in \reals^{F\times F}$ is a linear transformation to map the features between layers.
	The type of product graph rules the attention coefficients $\alpha_{ij}$ and their calculation.
	
	In both MPNN and GAT models, we can split the rule over space and time to have a different message passing or attention scheme.
	All in all, these represent different principles to learn spatiotemporal representation with product graphs. We continue with convolutional principle since it provides theoretical tools for their analysis and leave the MPNNs and GATs to interested readers.

\begin{figure}[!t]
	\centering

\def \thisplotscale {1.77}
\def \unit {\thisplotscale cm}

\tikzstyle {Phi} = [rectangle, 
thin,
minimum width = 1.15*\unit, 
minimum height = \sumshift*\unit, 
anchor = west,
draw,
fill = cyan!20]

\tikzstyle {sum} = [circle, 
thin,
minimum width  = 0.3*\unit, 
minimum height = 0.3*\unit, 
anchor = center,
draw,
fill = cyan!20]

\def \deltax {1.6}
\def \deltay {1.2}
\def \deltagat {1.2}
\def \sumshift {0.4}

{\footnotesize\begin{tikzpicture}[x = 1*\unit, y = 1*\unit]
		
		\pgfdeclarelayer{bg}    
		\pgfsetlayers{bg,main}  

		\node (first) [] {};    
		
		\path (first) ++ (0.30*\deltax, 0) node (0) [Phi] {$\bbH_1(\bbS_\diamond)$};
		\path (0)     ++ (0.75*\deltax, 0) node (1) [Phi] {$\sigma(\cdot)$};
		\path (first) ++ (0.30*\deltax, -0.6*\deltax) node (2) [Phi] {$\bbH_2(\bbS_\diamond)$};
		\path (2)     ++ (0.75*\deltax, 0) node (3) [Phi] {$\sigma(\cdot)$};
		\path (first) ++ (0.30*\deltax, -0.6*2*\deltax) node (4) [Phi] {$\bbH_3(\bbS_\diamond)$};
		\path (4)     ++ (0.75*\deltax, 0) node (5) [Phi] {$\sigma(\cdot)$};
		
		\path (5.east) ++ (1*\sumshift*\deltax, 0) node [anchor=west] (last) [] {};
		
		\path (1) ++ (0, -0.3*\deltax) coordinate [] (1c) [] {};
		\path (0) ++ (0, -0.3*\deltax) coordinate [] (2c) [] {};
		\path (3) ++ (0, -0.3*\deltax) coordinate [] (3c) [] {};
		\path (2) ++ (0, -0.3*\deltax) coordinate [] (4c) [] {};

		\path[-stealth] (first) edge [above, pos=-0.5] node {$\bbx_{\diamond,0}=\bbx_\diamond$}                 (0);	
		\path[-stealth] (0)     edge [above right, pos=0.0] node {}   (1);	
		\path[-stealth] (2)     edge [above right, pos=0.0] node {}   (3);	
		\path[-stealth] (4)     edge [above right, pos=0.0] node {}   (5);		
		
		\path[] (1) edge [above, pos=0.45] node {} (1c);
		\path[-stealth] (2c) edge [above left, pos=0.45] node {$\bbx_{\diamond,1}$} (2);
		\path[] (2c) edge [above, pos=0.45] node {} (1c);
		
		\path[] (3) edge [above, pos=0.45] node {} (3c);
		\path[-stealth] (4c) edge [above left, pos=0.45] node {$\bbx_{\diamond,2}$} (4);
		\path[] (4c) edge [above, pos=0.45] node {} (3c);

		\path[-stealth] (5) edge [above right, pos=0.1] node 
		{$\Phi(\bbx_\diamond;\ccalH(\bbS_\diamond))$} ++ (0.8*\deltax, 0);
		
		\begin{pgfonlayer}{bg}
			\node[fit=(0)(1), draw,inner sep = 4*\deltax,fill = red!30, fill opacity = 0.5,draw = white, rounded corners=10](l1) {};
			\node[fit=(2)(3), draw,inner sep = 4*\deltax,fill = red!30, fill opacity = 0.5,draw = white, rounded corners=10](l2) {};
			\node[fit=(4)(5), draw,inner sep = 4*\deltax,fill = red!30, fill opacity = 0.5,draw = white, rounded corners=10](l3) {};
		\end{pgfonlayer}

		\node[below] at (l1.north) {\tiny Layer 1};
		\node[below] at (l2.north) {\tiny Layer 2};
		\node[below] at (l3.north) {\tiny Layer 3};
\end{tikzpicture}}
	\caption{GTCNN block diagram. On each layer a graph time convolutional filter [cf.~\eqref{eq:graphtimefilter}] is composed by a point-wise non-linearity, and all these layers are cascaded to generate the output.}
	\label{fig:GTCNN_bd}
\end{figure}
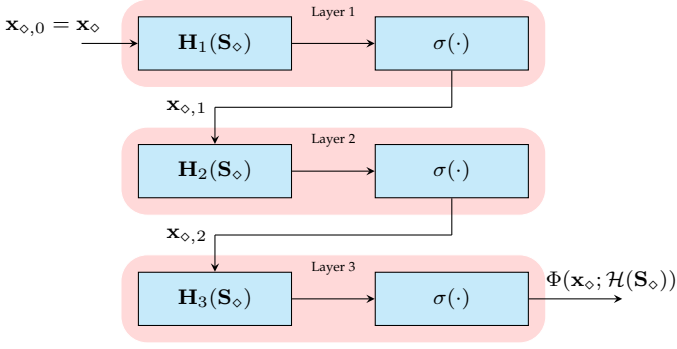

\subsection{Learning the Spatiotemporal Coupling}\label{subsec_learnSpatioTemp}

The GTCNN output in \eqref{eq:GTCNN} is influenced by the type of product graph but it is unclear which form is most suitable for a specific problem. To avoid imposing a wrong inductive bias, we consider the parametric product graph GSO
\begin{align}
\label{eq:generalgtfilter}
\mathbf{S}_{\diamond}= \sum_{i = 0}^1 \sum_{j = 0}^1 s_{ij}  \left(\mathbf{S}_{T}^{i} \otimes  \mathbf{S}^{j}  \right)
\end{align}
where by learning the four scalars $\{s_{ij}\}$ we learn the spatiotemporal coupling. The parametric product graph generalizes the three product graphs seen in Section~\ref{subsec_PG}. For instance, setting $s_{11}=1$ and the rest zero we get the Kronecker product. But allowing each $s_{ij}$ to take any real value weigh accordingly the impact of the spatial and temporal graphs into the final product graph. Note also that $s_{00} \neq 0$ implies the presence of spatiotemporal self-loops. If all $s_{ij}$'s are non-zero, the parametric product graph has $|\ccalE_\diamond| = |\ccalE_\boxtimes|+NT$ edges, which are $NT$ more edges than the strong product because of these self-loops. 

Building a GTCNN with the graph-time convolutional filter~\eqref{eq:graphtimefilter} but with the GSO~\eqref{eq:generalgtfilter} matches the spatiotemporal coupling with the task at hand. However, treating $\{s_{ij}\}$ as learnable parameters is practical only if we fix the filter order to $K = 1$.\footnote{This would be the spatiotemporal form of GCN~\cite{GCN_Welling}.}
For higher orders $K \ge 2$ this implies pre-computing the powers $\mathbf{S}_{\diamond}^k$ of~\eqref{eq:generalgtfilter} and storing them in the memory as the computational complexity of each multiplication is of order $\ccalO((NT)^3)$.
Yet, storing the powers in the memory is still impractical for large spatial graphs.
Another way to mitigate these issues is to rearrange the whole filter expression as shown next.

\begin{proposition}
	\label{prop1}
	Consider the spatial graph $\ccalG$ with shift operator $\bbS$ and the temporal graph $\ccalG_T$ with shift operator $\bbS_T$. The graph-time convolutional filter in~\eqref{eq:graphtimefilter} operating with the parametric product graph in~\eqref{eq:generalgtfilter} is a particular case of
	\begin{equation}
	\bbH(\bbS,\bbS_T) =
	\sum_{k=0}^{\bar{K}}
	\sum_{l=0}^{\tilde{K}}
	h_{kl} (\bbS_T^l \otimes \bbS^k)
	\label{eq:graphtimegeneralfilter}
	\end{equation}
where $\bar{K}$ and $\tilde{K}$ are the orders over the spatial and temporal graphs respectively and $\{h_{kl}\}$ are the parameters.
\end{proposition}
\begin{proof}
	See Appendix~\ref{sec:gtfilter_proof}.
\end{proof}

Building a GTCNN with a filter tensor $\ccalH(\bbS, \bbS_T) = \{	\bbH_\ell^{fg}(\bbS,\bbS_T)	\}_{\ell fg}$ yields an architecture that is more flexible than \eqref{eq:GTCNN}. This is because of:
\begin{enumerate}
\item Filter $\bbH(\bbS,\bbS_T)$ [cf. \eqref{eq:graphtimegeneralfilter}] has more freedom to control the spatial and temporal resolution in each layer through orders $\bar{K}$ and $\tilde{K}$. Differently, filter $\bbH(\bbS_\diamond)$ [cf. \eqref{eq:graphtimefilter}] does not allow for such an independent control, in which order $K$ influences the spatiotemporal resolution in a coupled manner. This allows exploiting more resolution in a particular domain.
\smallskip
\item Filter $\bbH(\bbS,\bbS_T)$ learns how the spatiotemporal coupling influences the graph-time convolution through parameters $\{h_{kl}\}$. This happens for each layer and feature. This is also more powerful than learning $\{s_{ij}\}$ and $\{h_k\}$ disjointly as in $\bbH(\bbS_\diamond)$ because it matches the importance of the multihop resolutions with the spatiotemporal links.
\smallskip
\item Filter $\bbH(\bbS,\bbS_T)$ enjoys the same computational cost of $\bbH(\bbS_\diamond)$. Computing the output $\bby_\diamond = \bbH(\bbS,\bbS_T)\bbx_\diamond$ requires computing all shifts $\bbx_\diamond^{(k,l)} = (\bbS_T^l \otimes \bbS^k) \bbx_\diamond$ for all $k\in [\bar{K}]$ and $l\in [\tilde{K}]$.
These shifts can be written in a recursive manner as
\begin{equation}
	\bbx_\diamond^{(k,l)} = (\bbS_T \bbS_T^{l-1} \otimes \bbS \bbS^{k-1}) \bbx_\diamond.
\end{equation}
Exploiting the properties of the Kronecker product\footnote{$(\bbA \otimes \mathbf{B})(\mathbf{C} \otimes \mathbf{D}) = \bbA \mathbf{C} \otimes \mathbf{B}\mathbf{D}$~\cite{petersen2008matrix}.}, we can write the latter as
\begin{align}
\bbx_\diamond^{(k,l)} &= (\bbS_T\otimes\bbS)(\bbS_T^{l-1} \otimes \bbS^{k-1})\bbx_\diamond \nonumber \\
&= (\bbS_T \otimes \bbI_N)(\bbI_T \otimes \bbS)(\bbS_T^{l-1} \otimes \bbS^{k-1})\bbx_\diamond.
\end{align}
Thus, we can compute $\bbx_\diamond^{(k,l)}$ recursively as
\begin{align}
\bbx_\diamond^{(k,l)} &=  (\bbS_T \otimes \bbI_N)\bbx_\diamond^{(k,l-1)}\nonumber\\
&= (\bbS_T \otimes \bbI_N)(\bbI_T \otimes \bbS)\bbx_\diamond^{(k-1,l-1)}
\label{eq:garphtimerecursion}
\end{align}
where $\bbx_\diamond^{(k,l-1)} = (\bbI_T \otimes \bbS)\bbx_\diamond^{(k-1,l-1)}$ is the spatially shifted signal and $\bbx_\diamond^{(0,0)} \coloneqq \bbx_\diamond$. Recursion~\eqref{eq:garphtimerecursion} implies that we can compute $\bbx_\diamond^{(k,l)}$ from $\bbx_\diamond^{(k-1,l-1)}$ with a cost of $\ccalO(T|\ccalE| + N|\ccalE_T|)$. Since we need to perform the latter for all $k\in [\bar{K}]$ and $l\in [\tilde{K}]$, we have a computational cost of order $\ccalO(\bar{K}T|\ccalE| + \tilde{K}N|\ccalE_T|)$, which is linear in the product graph dimension. The GTCNN cost is $F^2L$ times the latter.
\end{enumerate}

\subsection{Properties}\label{subsec_prop}

The convolution principle allows studying two fundamental properties: equivariances to permutations in the vertex domain and the filters' frequency response in the frequency domain. We shall discuss these properties for a GTCNN with the more general filter $\bbH(\bbS,\bbS_T)$ [c.f.~\eqref{eq:graphtimegeneralfilter}] but the results hold also for filter $\bbH(\bbS_\diamond)$ [c.f.~\eqref{eq:graphtimefilter}].

\medskip
\noindent\textbf{Permutation equivariance}
allows assigning an arbitrary ordering to the nodes.
This is a desirable property since a graph does not change by the permutation of its nodes, hence the GTCNN output should stay unchanged up to a reordering.
The GTCNN is permutation equivariant as shown by the following proposition.
\begin{proposition}
	\label{prop2}
	Let $\bbx_\diamond = \textnormal{vec}(\bbX)$ be the signal over the product graph $\ccalG_\diamond$ with shift operator $\bbS_\diamond = \bbS_T \diamond \bbS$. Consider also the output of a graph-time filter $\bbH(\bbS_T,\bbS)\bbx_{\diamond}$. Then, for a permutation matrix $\bbP$ belonging to the set
	\begin{equation}
		\ccalP = \{\bbP = \{0,1\}^{N\times N} \colon \bbP\bbone=\bbone, \bbP^\top\bbone=\bbone\} \nonumber
	\end{equation}
	it holds that
	\begin{equation}
		\bbP\!^\top\!\textnormal{vec}\!^{-\!1}\!(\Phi(\bbx_\diamond;\ccalH(\bbS_T,\bbS)))\! =\!\Phi(\bbx_\diamond;\ccalH(\bbS_T,\bbP\!^\top\!\bbS\bbP))\textnormal{vec}(\bbP\!^\top\!\bbX).
	\end{equation}
\end{proposition}
\begin{proof}
	See Appendix.~\ref{sec:permequi_proof}.
\end{proof}

That is, the GTCNN operating on a multivariate time-series over a spatial graph $\bbS$ has the equally permuted output of the GTCNN operating on permuted spatial graph $\bbP^\top\bbS\bbP$.
This property allows the GTCNN to exploit the spatiotemporal data dependencies and symmetries encoded in the product graph to enhance learning~\cite{bronstein2021geometric}.

\medskip
\noindent\textbf{Spectral analysis}
of filter \eqref{eq:graphtimegeneralfilter} can help us to study their behaviors and what spectral features they learn for a specific task.
It also provides a fundamental framework for further analyzing the robustness of the GTCNNs to perturbations as we shall detail in the next section.
To do so, consider the eigendecomposition of the spatial GSO $\bbS= \bbV \bbLambda \bbV^{\herm}$ with eigenvevtors $\bbV = [\bbv_1, \ldots, \bbv_N]$ and eigenvalues $\bbLambda = \diag(\lambda_1,\dots,\lambda_N)$.
Consider also the eigendecomposition of the temporal GSO $\bbS_T= \bbV_T \bbLambda_T \bbV_T^{\herm}$ where $\bbV_T=[\bbv_{T,1},\dots,\bbv_{T,T}]$ is the matrix of temporal eigenvectors and $\bbLambda_T = \diag(\lambda_{T,1},\dots,\lambda_{T,T})$ is that of temporal eigenvalues. 
Then, the GSO of the product graph $\bbS_\diamond$ can be eigendecomposed as
\begin{equation}
\bbS_\diamond= \bbV_\diamond \bbLambda_\diamond\bbV_\diamond^\herm
= (\bbV_T \otimes \bbV)(\boldsymbol{\Lambda}_T \diamond \boldsymbol{\Lambda})(\bbV_T \otimes \bbV)^{\herm}
\end{equation}
with eigenvectors are $\bbV_\diamond=\bbV_T \otimes \bbV$ and eigenvalues $\bbLambda_\diamond = \boldsymbol{\Lambda}_T \diamond \boldsymbol{\Lambda}$ dictated by the product graph.
Then, we can define the graph-time Fourier transform of signal $\bbx_\diamond$ as $\tilde{\bbx}_{\diamond} =  (\bbV_T \otimes \bbV)^\herm \bbx_\diamond $~\cite{grassi2017time,sandryhaila2014big}.
This transform represents the variation of the multivariate time-series over the product graph in terms of product graph eigenvectors.
Using these concepts, the following proposition characterizes the spectral behavior of filter~\eqref{eq:graphtimegeneralfilter}.

\begin{proposition}
	\label{prop3}
	Consider the eigendecomposition of the spatial GSO $\bbS = \bbV\bbLambda\bbV^\herm$ and the temporal GSO $\bbS_T = \bbV_T\bbLambda_T\bbV_T^\herm$.
	Consider also the graph-time Fourier transform of the output signal $\tilde{\bby}_\diamond = (\bbV_T \otimes \bbV)^\herm \bby_{\diamond}$ and the input signal  $\tilde{\bbx}_\diamond = (\bbV_T \otimes \bbV)^\herm \bbx_{\diamond}$.
	Then, the filtering operation~\eqref{eq:graphtimegeneralfilter} operates in the graph-time frequency domain as
	\begin{equation}
		\tilde{\bby}_{\diamond}= h\left(\boldsymbol{\Lambda}_T, \boldsymbol{\Lambda}\right) \tilde{\bbx}_\diamond
	\end{equation}
	with the filter frequency response matrix
	\begin{equation}
		h\left(\boldsymbol{\Lambda}_T, \boldsymbol{\Lambda}\right)
		= \sum_{k=0}^{\bar{K}} \sum_{l=0}^{\tilde{K}}h_{kl}(\bbLambda_T^l \otimes \bbLambda^k)
		\label{eq:spectral}
	\end{equation}
	with diagonal entries  $\left[h\left(\boldsymbol{\Lambda}_T, \boldsymbol{\Lambda}\right)\right]_{k k}=h\left( \lambda_{Tt}, \lambda_{i}\right)$ and $k=N(t-1)+i$ for $i=1, \ldots, N$ and $t=1, \ldots, T$.
\end{proposition}
\begin{proof}
	See Appendix.~\ref{sec:spectral}.
\end{proof}

That is, the graph-time filtered version $\bby_{\diamond}$ of $\bbx_{\diamond}$ with the filter $\bbH(\bbS_T,\bbS)$ corresponds to an element-wise multiplication in the graph-time frequency representation, i.e., $\hat{\bby}_{ti} = h([\bbLambda_T]_t,[\bbLambda]_i)\hat{\bbx}_{ti}$.
This is a direct extension of the convolution theorem~\cite{oppenheim1999discrete} to the spatiotemporal setting, ultimately, justifying the qualifier convolution for filter~\eqref{eq:graphtimegeneralfilter}. It shows that by learning the parameters ${h_{kl}}$ we not only learn spatiotemporal coupling in the vertex domain with desirable properties but also the frequency response of these filters to extract relevant spectral patterns.

\section{Stability Analysis}
\label{sec_stability}
We now investigate the stability properties of GTCNNs w.r.t. perturbations on the spatial graph to characterize its learning capabilities.
Analyzing stability is important as we often may not have access to the ground truth spatial graph.
In some cases, the inferred graph may be imperfect and we have to train the GTCNN over a perturbed graph;
in other cases, practical physical networks (e.g., water or power networks) slightly differ from the one we train the GTCNN, e.g., because of model mismatches.
So, having a stable GTCNN is desirable to perform the task reliably and allow transference~\cite{Ruiz2021GNN}.

We analyze the stability w.r.t. the \emph{relative perturbation} over the spatial graph
\begin{equation}
	\hat{\bbS} = \bbS + (\bbE\bbS + \bbS\bbE)
	\label{relperturb}
\end{equation}
where $\hat{\bbS}$ is the perturbed GSO and $\bbE$ is the perturbation matrix~\cite{gamamain}.
Such model suggests that the graph perturbation depends on its structure, i.e., a node with more connected edges is relatively more prone to more perturbation.  
To characterize the GTCNN stability,
we consider graph-time convolutional filters with a frequency response that varies slightly in the high spatial frequencies (eigenvalues $\bbLambda$) and arbitrarily in the temporal one, see Fig.~\ref{fig:Lipschitz}. 
\begin{definition}
	\label{def:intlip}
	Given a graph-time filter with an analytic frequency response $h(\lambda_T,\lambda)$ [cf.~\eqref{eq:spectral}]. We say this filter is graph integral Lipschitz if there exists constant $C>0$ such that for all graph frequencies $\lambda_1,\lambda_2 \in \bbLambda$, it holds that
	\begin{equation}
		|h(\lambda_T,\lambda_2)-h(\lambda_T,\lambda_1)| \leq C \frac{|\lambda_2 - \lambda_1|}{|\lambda_2 + \lambda_1|/2}.
		\label{eq:lipmain}
	\end{equation}
\end{definition}
Expression~\eqref{eq:lipmain} implies that the filter's two-dimensional frequency response is Lipschitz in any interval $(\lambda_1, \lambda_2)$ of graph frequencies where the Lipschitz constant depends on their gap $|\lambda_2 + \lambda_1|/2$. The latter is similar to the integral Lipschitz property for GCNNs working on time invariant signals~\cite{gamamain} since we treat spatial perturbations, but it does not restrict the temporal behavior of the filter.%
\footnote{If we would ignore the product graph structure and discuss the stability results using~\cite{gamamain} straightforwardly the filter needed be spatiotemporal integral Lipschitz. This in turn affects discriminability~\cite{Ruiz2021GNN}.}
For the two-dimensional frequency response, this implies that its partial derivative is restricted as 
\begin{equation}
	\left|\lambda\frac{\partial h(\lambda_T,\lambda)}{\partial \lambda}\right| \leq C,	\label{eq:lipsch}
\end{equation}
i.e.,
the filters in a GTCNN have a frequency response that cannot vary drastically on high graph frequencies for all temporal frequencies but the filter can vary arbitrary over the temporal frequencies.
Fig.~\ref{fig:Lipschitz} illustrates this property.

\begin{definition}
	\label{def:norm}
	Consider a graph-time convolutional filter with an analytic frequency response $h(\lambda_T,\lambda)$ [cf.~\eqref{eq:spectral}]. We say this filter has a normalized spectral response if $|h(\lambda_T,\lambda)| \leq 1$ for all $\lambda$, $\lambda_T$.
\end{definition}
This implies that the filter's gain $B = \|\bbH(\bbS_T,\bbS)\|$ to be less or equal to 1, $B\leq 1$, w.r.t. an $l_2$-measure.
Def.~\ref{def:norm} is required as the output of each layer should not magnify the input norm, otherwise, the system will become less stable as the number of layers increase, even in absence of perturbation.
Otherwise, constant $B$ would also appear in the following stability result of the GTCNN.


\begin{figure}[!t]
	\centering
	\includegraphics[width=0.7\linewidth]{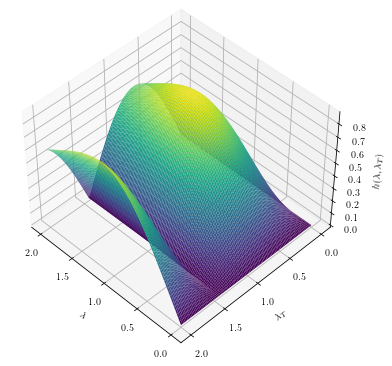}
	\caption{A spatial Lipschitz graph-time filter. The frequency response varies smoothly over high graph frequencies, while it has sudden changes on temporal frequencies.}
	\label{fig:Lipschitz}
\end{figure}

\begin{theorem}
	\label{filterstanility}
	Let $\bbS = \bbV\bbLambda\bbV^\herm$ and $\bbS_T\bbV_T\bbLambda_T\bbV_T^\herm$ be spatial and temporal graph shift operators, respectively.
	Let also $\hat{\bbS}$ be the relatively perturbed graph shift operator [cf.~\eqref{relperturb}].
	Consider the error matrix has the eigendecomposition
	$\bbE = \bbU \bbM \bbU^\herm$
	where $\bbU$ are the eigenvectors and $\bbM$ is the diagonal matrix of eigenvalues. Assume the error matrix has a bounded operator norm $\|\bbE\|\leq \eps.$
	Consider a GTCNN with L layers, F features and integral Lipschitz spatial-temporal graph filters [cf. Def.~\ref{def:intlip}] with normalized spectral responses [cf. Def.~\ref{def:norm}].
	Let also its nonlinearities be 1-Lipschitz, i.e., $|\sigma(a)-\sigma(b)| < |a-b|$, such as ReLU.
	Then, the distance between the GTCNN outputs $\Phi(\bbx_\diamond;\bbS_T,\bbS,\ccalH)$ on the nominal graph and $\Phi(\bbx_\diamond;\bbS_T,\hat{\bbS},\ccalH)$ on the perturbed graph is upper bounded by
	\begin{align}
	\|\Phi(\bbx_\diamond;\ccalH(\bbS_T,\bbS)) - \Phi(\bbx_\diamond;\ccalH(\bbS_T,\hat{\bbS})) \|_2 \leq LF^{L-1}\Delta\eps\|\bbx_\diamond\|_2
	\label{stabilitybound}
	\end{align}
	where $\Delta = 2C(1+\delta T\sqrt{N})$, $\delta = (\|\bbU-\bbV\|^2 +1)^2-1$ indicates the eigenvector misalignment between the spatial graph shift operator $\bbS$ and error matrix $\bbE$, $N$ is the size of spatial graph, and $T$ is the size of temporal graph.
\end{theorem}
\begin{proof}
	See Appendix~\ref{sec:stability_proof}.
\end{proof}


Result~\eqref{stabilitybound} generalizes the findings in \cite{gamamain} to the spatiotemporal domain and states that GTCNNs are stable to relative perturbations in the spatial graph.
Together with the permutaiton equivariance~[Prop.~2], results~\eqref{stabilitybound} shows that GTCNNs allow for inductive learning and that are transferable architectures.
Such result provides three main insight on its stability/transferability:
\begin{enumerate}
\item The GTCNN is less stable for larger graphs ($\sqrt{N}$) as more nodes are exchanging information over a perturbed graph. Instead, its stability is more affected by the temporal resolution as increasing $T$ implies replicating the entire perturbed spatial graph.
\item The GTCNN is less stable if it is more discriminative in the nominal graph. This is reflected by term  $LF^{L-1}$ and it is due to the larger number of filters operating on the perturbed graphs. Such a observation shows also an inherit trade-off between stability and discriminability of GTCNNs.
\item Finally, we see the impact of each individual filter via the Lipschitz constant $C$. The latter in turn controls the filter discriminability on high graph frequencies to improve stability.
\end{enumerate}
%
%
%


\smallskip
\noindent\textbf{Comparison with alternative bounds.} We discuss now that exploring the product graph sparsity with GTCNN provides more stable solutions compared with baselines that ignore it.
We discuss also the relation of result \eqref{stabilitybound} with the stability result for spatiotemporal learning~\cite{mo:ruiz2020gatedGRNN,hadou2021space}.


\smallskip
\emph{Product graph GCNN:} One trivial learning solution on product graphs is to ignore what different edges represent and naively deploy a GCNN over this large graph of $NT$ nodes. Using the results in
~\cite{gamamain}, such GCNN has a stability bound $\Delta_\text{GCNN} = 2C\sqrt{T}(1+\delta\sqrt{NT})$ since it assumes the perturbation over all the edges in the product graph including temporal ones, so, the noise energy scales by $\sqrt{T}$ and it applies over $NT$ nodes.
The latter leads to a stability bound $2CT$ times looser than~\eqref{stabilitybound}.
Moreover, Theorem~\ref{filterstanility} restrains the graph-time filter variability only on spatial frequencies, while if we apply directly the results of~\cite{gamamain}, we restrain also high temporal frequency variations.

\smallskip
\emph{GCNNs:}
Another way to approach spatiotemporal learning is to treat the time series as features over the nodes and deploy a conventional GCNN. Leveraging again the result of~\cite{gamamain}, such a solution will have a stability bound $\Delta_\text{GCNN} = 2CT^L(1+\delta\sqrt{N})$ as we replicate filters in each layer $T$ times for each feature, including the input layer.
This bound is magnificently large due to factor $T^L$.

\smallskip
\emph{GGRNN~\cite{mo:ruiz2020gatedGRNN}:}
Graph gated recurrent neural networks (GGRNNs) replace the linear transformations in a recurrent neural network with graph filters to learn spatiotemporal patterns. They have a stability bound $\Delta_{\text{GGRNN}} = C(1+\sqrt{N}\delta)(T^2+3T)$. 
Comparing with~\eqref{stabilitybound}, $\Delta_{\text{GGRNN}}$ expands at a higher rate by a factor of $T$, but note that the term $T^2+3T$ also implicitly contains the effect of layers $L$ in itself.
The assumptions behind the stability theorem in~\cite{mo:ruiz2020gatedGRNN} are similar to Theorem~\ref{filterstanility}, i.e., spatial integral Lipschitz filters and 1-Lipschitz nonlinearity.
In conclusion, Theorem~\ref{filterstanility} and~\cite{mo:ruiz2020gatedGRNN} present closely related stability bounds based similar conditions but for different models.

\smallskip
\emph{ST-GNN~\cite{hadou2021space}:}
Space-time graph neural network (ST-GNN) linearly composes a GSO with a continuous time shift operator to define a space-time shift operator.
The stability to relative perturbation on the spatial graph is equal to GCNN~\cite{gama2020elvinmagazine}, i.e., $\Delta_{\text{STGNN}} = 2C(1+\delta\sqrt{N}) $.
This bound suggests that ST-GNN can process temporal properties of times series robustly under spatial perturbation.
The result~\eqref{stabilitybound} reflects the effect of time series length on stability of GTCNN as it involves time directly into convolutional filters to capture spatiotemporal patterns in the data.

\section{Numerical Results}
\label{sec_numeric}
This section evaluates the GTCNN performance in different scenarios and compares it with other state-of-the-art algorithms. In all experiments, the ADAM optimizer is used to train the model and an unweighted directed line graph is selected as the temporal graph.

\subsection{Source Localization}
The task is to detect the source of a diffusion process over a graph by observing the time-series of length $T$.
The graph is an undirected stochastic block model with $C = 5$ communities and $N = 100$ nodes.
The edges are randomly and independently drawn with probability $0.8$ for nodes in a same community and $0.2$ for nodes in different communities.
At $t=0$, a random node takes a unitary value and diffuses it throughout the network 30 times as $e^{-t\bbL}$.
The model is fed with a multivariate time series $\{\bbx_{t-T},\dots,\bbx_{t-1}\}$ randomly selected after at least $t = 15$ diffusion and the goal is to detect the community corresponding to the source node.  
Our aim here is to study the role of the different product graphs; hence, we compare different GTCNNs with the baseline GCNN that ignores the temporal connections and treats time as feature.


The dataset contains $2000$ samples with an $80/10/10$ split.
All architectures have two layers with two second order filters.
The cross-entropy loss is used for all models except for the parametric which its cost is also regularized by an $l_1$-norm of product graph parameters $\{s_{ij}\}$ with regularization weight $\beta = 0.05$ to enforce sparse spatiotemporal connections.
The temporal windows are selected from $T\in\{2,3,4,5\}$ and the features from $F_1,F_2\in\{4,8,16,32\}$.
The models are trained for $1000$ epochs with $100$ batch size.
Each experiment is done for $10$ different random graphs and $10$ different dataset realization.

\begin{figure}[t]
	\centering
	\includegraphics[width=0.9\linewidth]{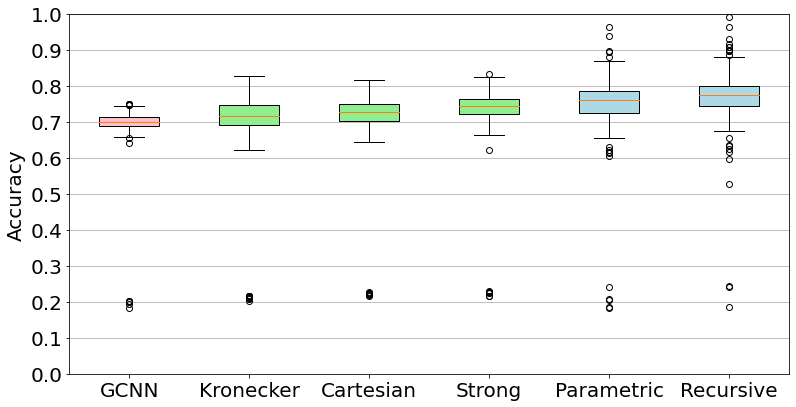}
	\caption{Comparison of GCNN with parametric and non-parametric GTCNN performances on source localization task to emphasize on the role of spatiotemporal couplings.}
	\label{fig:exp1}
\end{figure}

The results in Fig.~\ref{fig:exp1} suggest that accounting for the temporal connections even through a fixed product graph improves the performance.
Considering consecutive times as features in the GCNN can be translated as a fully connected temporal graph, so, the improved performance by enforcing a reasonable structure to temporal samples was expected.  
Better results are achieved by using parametric product graph as it learns the temporal connections for the specific task.
This flexibility reduces also the number of failed training attempts compared with the fixed product graphs and the GCNN.

\subsection{Multivariate Time-Series Forecasting}

\begin{table*}[t]
	\centering
	\caption{Performance comparison of GTCNN and other baseline models for different prediction horizons. The best performance is shown in bold and the second best is underlined.
	  The standard deviation of all models are of the order $10^{-3}$ and are omitted to avoid an overcrowded table.}
	\label{tablemain}
	\renewcommand{\arraystretch}{1.5}
	\begin{tabular}{|ll|lll|lll|lll|}
		\hline
		\multirow{2}{*}{Data}     & \multirow{2}{*}{Models} & \multicolumn{3}{c|}{3-Steps}                    & \multicolumn{3}{c|}{6-Steps}                    & \multicolumn{3}{c|}{12-Steps}                    \\ \cline{3-11} 
		&                         & MAE           & RMSE          & MAPE            & MAE           & RMSE          & MAPE            & MAE           & RMSE          & MAPE             \\ \hline
		\multirow{8}{*}{\begin{sideways}METR-LA\end{sideways}}
		& ARIMA~\cite{li2017diffusion}                   
		& 3.99          & 8.21          & 9.60\%          & 5.15          & 10.45         & 12.70\%         & 6.90          & 13.23         & 17.40\%          \\
		& G-VARMA~\cite{mo:isufi2019VARMA}                 
		& 3.60          & 6.89          & 9.62\%          & 4.05          & 7.84          & 11.22\%         & 5.12          & 9.58          & 14.00\%          \\
		& GP-VAR~\cite{mo:isufi2019VARMA}                  
		& 3.56          & 6.54          & 9.55\%          & 3.98          & 7.56          & 11.02\%         & 4.87          & 9.19          & 13.34\%          \\
		& FC-LSTM~\cite{li2017diffusion}                  & 3.44          & 6.30          & 9.60\%          & 3.77          & 7.23          & 10.90\%         & 4.37          & 8.69          & 13.20\%          \\
		& Graph Wavenet~\cite{mo:wu2019graph}           
		& \ul{2.69}     & \textbf{5.15} & \ul{6.90\%}     & \ul{3.07}     & \ul{6.22}     & \ul{8.37\%}     & \textbf{3.53} & \ul{7.37}     & \textbf{10.01\%} \\
		& STGCN~\cite{mo:yu2017spatio}                   
		& 2.88          & 5.74          & 7.62\%          & 3.47          & 7.24          & 9.57\%          & 4.59          & 9.40          & 12.70\%          \\
		& GGRNN~\cite{mo:ruiz2020gatedGRNN}                  
		& 2.73          & 5.44          & 7.12\%          & 3.31          & 6.63          & 8.97\%          & 3.88          & 8.14          & 10.59\%          \\
		& GTCNN (this work)                   & \textbf{2.68} & \ul{5.17}     & \textbf{6.85\%} & \textbf{3.02} & \textbf{6.20} & \textbf{8.30\%} & \ul{3.55}     & \textbf{7.35} & \ul{10.21\%}          \\ \hline
		\multirow{8}{*}{\begin{sideways}PEMS-BAY\end{sideways}}
		& ARIMA~\cite{li2017diffusion}                  
		& 1.62          & 3.30          & 3.50\%          & 2.33          & 4.76          & 5.40\%          & 3.38          & 6.50          & 8.30\%           \\
		& G-VARMA~\cite{mo:isufi2019VARMA}                 
		& 1.88          & 3.96          & 4.28\%          & 2.45          & 4.70          & 5.42\%          & 3.01          & 5.83          & 7.10\%           \\
		& GP-VAR~\cite{mo:isufi2019VARMA}                  
		& 1.74          & 3.22          & 3.45\%          & 2.16          & 4.41          & 5.15\%          & 2.48          & 5.04          & 6.18\%           \\
		& FC-LSTM~\cite{li2017diffusion}                   & 2.05          & 4.19          & 4.80\%          & 2.20          & 4.55          & 5.20\%          & 2.37          & 4.74          & 5.70\%           \\
		& Graph Wavenet~\cite{mo:wu2019graph}           
		& \ul{1.30}     & \ul{2.74}     & \ul{2.73\%}     & \textbf{1.63} & \ul{3.70}     & \textbf{3.67\%} & \textbf{1.95} & \textbf{4.52} & \textbf{4.63\%}  \\
		& STGCN~\cite{mo:yu2017spatio}                   
		& 1.36          & 2.96          & 2.90\%          & 1.81          & 4.27          & 4.17\%          & 2.49          & 5.69          & 5.79\%           \\
		& GGRNN~\cite{mo:ruiz2020gatedGRNN}                   
		& 1.33          & 2.81          & 2.83\%          & 1.68          & 3.94          & \ul{3.79\%}     & 2.34          & 5.14          & 5.21\%           \\
		& GTCNN (this work)                   & \textbf{1.25} & \textbf{2.66} & \textbf{2.61\%} & \ul{1.65}     & \textbf{3.68} & 3.82\%          & \ul{2.27}     & \ul{4.99}     & \ul{5.11\%}      \\ \hline
	\end{tabular}
	\renewcommand{\arraystretch}{1}	
\end{table*}

We applied the GTCNN to address traffic and weather forecasting on four benchmark datasets.
On the traffic forecasting task, we used parametric product graph while the recursive model is applied for the weather forecasting problem.
The goal is to show that both solutions compare well with alternatives.
The results are compared with baseline methods to provide insights into the GTCNN capabilities and limitations. 
For baselines, we considered:
\begin{itemize}
	\item \textit{ARIMA}:
	auto-regressive integrated moving average model using Kalman filter~\cite{li2017diffusion}. This model treats each time series individually.
	\item \textit{G-VARMA}:
	graph vector auto-regressive moving average model~\cite{mo:isufi2019VARMA}. This model works upon statistical assumptions on the data and takes the spatial graph into account.
	\item \textit{GP-VAR}:
	graph polynomial auto-regressive model~\cite{mo:isufi2019VARMA}. It has fewer parameters than G-VARMA yet considers the spatial graph.
	\item \textit{FC-LSTM}~\cite{li2017diffusion}:
	fully connected LSTM performing independently on time series, i.e., one LSTM per time series.
	\item \textit{Graph WaveNet}:
	A hybrid convolutional model for time-series over graph~\cite{mo:wu2019graph}.
	\item \textit{STGCN}:
	spatial-temporal graph convolution network which uses graph convolution module alongside with 1D convolution~\cite{mo:yu2017spatio}.
	\item \textit{GGRNN}:
	gated graph recurrent neural network which replace linear transforms in a RNN by graph convolutional filters~\cite{mo:ruiz2020gatedGRNN}.
\end{itemize}

\medskip \noindent
\textbf{Traffic forecasting.}
We considered the traffic network datasets METR-LA and PEMS-BAY.
METR-LA contains four months of recorded traffic data over 207 nodes on the highways of Los Angeles County with 5 minutes resolution~\cite{li2017diffusion}.
PEMS-BAY includes six months of traffic information over 325 nodes in Bay Area with similar resolution of METR-LA.
We considered the same setting as in~\cite{shuman2013emerging}.
The shift operator is a directed adjacency matrix constructed by applying a Gaussian threshold kernel over the road network distance matrix.
The goal is to predict the traffic load in time horizons $15-30-60$ minutes having the time series for last 30 minutes, i.e., $T = 6$.

Both datasets are divided into an $80/10/10$ split chronologically.
The GTCNN is fixed and contains two layers with third order filters and the parametric product graph.
We evaluated the number of features in each layer from $F \in \{4,8,16\}$.
The objective function is the regularized mean squared error (MSE) via the $l_1$-norm on the product graph parameters $\bbs$, i.e., $\ccalL = \text{MSE}(\hat{\bbx}_{t+1},\bbx_{t+1}) + \beta \|\bbs \|_1$.
The regularization weight is chosen from $\beta \in \{0,0.05,0.1\}$.
For the GGRNN, we evaluated features $F \in \{4,8,16\}$ and filter orders $K \in \{3,4,5\}$.
For the other models, the parameters have been set similar to~\cite{mo:wu2019graph}.
The evaluation metrics are the mean absolute error (MAE), the root mean squared error (RMSE), and the mean absolute percentage error (MAPE).

Table~\ref{tablemain} compares the performance of GTCNN and other baseline models.
The GTCNN outperforms the other models in a short horizon while Graph WaveNet works better for longer horizons.
The benefits in the short term are due to high order spatiotemporal aggregation in the GTCNN which allows capturing efficiently the spatiotemporal patterns in the data.
On the longer term, the Graph WaveNet works better because it captures longer term patterns by increasing the receptive field of the model through dilated convolutions.
Graph Wavenet can also be fed with a longer time series as it works only with the spatial graph. 
Differently, the GTCNN does not capture them and an efficient implementation remains an interesting extension.


\medskip \noindent
\textbf{Weather forecasting.}
We considered two benchmark datasets, Molene and NOAA.
The Molene dataset contains 744 hourly temperature measurement across 32 stations in a region of France.
The NOAA dataset contains 8579 hourly temperature measurement across 109 stations in the U.S..
The same setting as~\cite{mo:isufi2019VARMA} is used in this experiment.
Our aim is to make a prediction of the temperature for $1-3-5$ hours ahead having the time series for last 10 hours.

The model is fixed and consists of two recursive GTCNN [c.f.~\eqref{eq:graphtimegeneralfilter}] layers with temporal locality selected from $\tilde{K} \in\{3,4,5\}$ and spatial locality from $\bar{K} \in \{3,5,7\}$.
For the neural network models, the number of features is chosen via grid search from $F\in\{4,8,16\}$.
The GGRNN filter orders are varied among $K \in \{3,4,5\}$.
For the statistical models we set the parameters similar to~\cite{mo:isufi2019VARMA}.
The number of LSTM hidden units are selected from $\{8,16,32,64\}$.
The loss function is the MSE at 1-step prediction.

Table~\ref{tableweather} indicates the model performance for different prediction horizons.
In the Molene dataset, graph-based statistical methods outperform the rest as the dataset contains fewer samples and the time series have clear patterns in their temporal variation.
Hence, leveraging a statistical assumption for the data compensates for the lack of samples and leads to a better performance.
Among these methods, G-VARMA performs the best which indicates the importance of inducing the graph in the model.
Due to the temporal connections, the GTCNN still performs better than the neural network counterparts.
In the NOAA dataset, the abundance of training data allows the neural network models to learn complicated patterns and outperform statistical-based models. All the neural network alternatives perform closely, however, in higher forecasting horizons LSTM starts to take over the other variants while GTCNN performs better in short horizons.
Overall, the GTCNN can be considered as a valid alternative for learning spatiotemporal representations in both cases where the training set is limited or large.

\begin{table*}[]
	\centering
	\caption{The performance of the GTCNN compared with baseline methods on the weather forecasting task. The best performance is shown in bold and the second best is underlined.
	The standard deviation of all models are of the order $10^{-2}$ for the Molene dataset and $10^{-3}$ for the NOAA dataset and they are omitted to avoid an overcrowded table.}
	\label{tableweather}
	\renewcommand{\arraystretch}{1.5}
	
\begin{tabular}{|ll|lll|lll|lll|}
	\hline
	\multirow{2}{*}{Data}     & \multirow{2}{*}{Models} & \multicolumn{3}{c|}{1-Steps}                    & \multicolumn{3}{c|}{3-Steps}                    & \multicolumn{3}{c|}{5-Steps}                    \\ \cline{3-11} 
	&                         & MAE           & RMSE          & MAPE            & MAE           & RMSE          & MAPE            & MAE           & RMSE          & MAPE             \\ \hline
	\multirow{8}{*}{\begin{sideways}Moelene\end{sideways}}
	& ARIMA~\cite{li2017diffusion}                   
	& 2.26          & 4.74         & 5.24\%          & 3.56          & 7.95         & 12.37\%         & 6.15          & 13.59         & 17.51\%          \\
	& G-VARMA~\cite{mo:isufi2019VARMA}                 
	& \textbf{2.09}         & \textbf{4.40}          & \textbf{4.97\%}          & \textbf{3.45}          & \textbf{7.69}         & \textbf{12.07\% }        & \textbf{6.00 }         & \textbf{13.26 }         & \textbf{17.42\%}          \\
	& GP-VAR~\cite{mo:isufi2019VARMA}                  
	& \ul{2.13}         & \ul{4.47}          & \ul{5.01\% }         & \ul{3.52}          & \ul{7.85}          & \ul{12.10\%}         & \ul{6.05}          & \ul{13.38}          & \ul{17.46\% }         \\
	& LSTM
	& 3.84          & 8.06          & 12.44\%          & 4.72          & 10.52          & 15.54\%         & 7.84         & 17.33          & 21.09\%          \\
	& Graph Wavenet~\cite{mo:wu2019graph}           
	& 3.52     & 7.41 & 11.39\%     & 4.49     & 10.02     & 15.11\%     & 7.54 & 16.66     & 20.14\% \\
	& STGCN~\cite{mo:yu2017spatio}                   
	& 3.78          & 7.95          & 15.20\%          & 4.56          & 10.18          & 9.57\%          & 7.50          & 16.59         & 20.12\%          \\
	& GGRNN~\cite{mo:ruiz2020gatedGRNN}                  
	& 3.17         & 6.67          & 15.20\%          & 4.55          & 10.16          & 8.97\%          & 7.35          & 16.26          & 19.89\%          \\
	& GTCNN (this work)
	& 3.55 & 7.47     & 11.45\% & 4.43 & 9.87 & 14.98\% & 6.54     & 14.46 & 18.53\%          \\ \hline
	\multirow{8}{*}{\begin{sideways}NOAA\end{sideways}}
	& ARIMA~\cite{li2017diffusion}                  
	& 0.39          & 0.98         & 1.68\%          & 1.60          & 3.86         & 4.21\%          & 2.89         & 6.75          & 8.47\%           \\
	& G-VARMA~\cite{mo:isufi2019VARMA}                 
	& 0.36         & 0.89         & 1.67\%          & 1.39         & 3.35          & 3.56\%          & 2.59         & 5.48          & 6.92\%           \\
	& GP-VAR~\cite{mo:isufi2019VARMA}                  
	& 0.41       & 1.02         & 1.70\%          & 1.40          & 3.37          & 3.56\%          & 2.59          & 6.04          & 7.31\%           \\
	& LSTM                    & 0.29          & 0.75          & 1.58\%          & \ul{0.66 }         & \ul{1.59}          & \ul{2.31\%}          & \textbf{1.39 }         & \textbf{3.24 }         & \textbf{3.46\% }          \\
	& Graph Wavenet~\cite{mo:wu2019graph}           
	& 0.33     & 0.83     & 1.65\%    & 0.95 & 2.28     & 2.84\%  &\ul{ 1.42} & \ul{3.31} & \ul{3.51\% } \\
	& STGCN~\cite{mo:yu2017spatio}                   
	& 0.35          & 0.87          & 1.65\%          & 0.82          & 1.97          & 2.44\%          & 1.68         & 3.91          & 3.92\%           \\
	& GGRNN~\cite{mo:ruiz2020gatedGRNN}                   
	& \textbf{0.26}          & \textbf{0.64}          & \textbf{1.54\%}          & 0.84          & 2.02          & 2.47\%     & 1.63          & 3.80         & 3.89\%           \\
	& GTCNN (this work)                   &\ul{ 0.28 }& \ul{0.70} & \ul{1.57\%} &\textbf{ 0.63}     & \textbf{1.53 } & \textbf{2.28\% }         & 1.53     & 3.57     & 3.65\%      \\ \hline
\end{tabular}

	\renewcommand{\arraystretch}{1}
\end{table*}

\subsection{Stability Analysis}

\begin{figure*}
	\centering
	\begin{subfigure}[b]{0.22\textwidth}
		\centering
		\includegraphics[width=\textwidth]{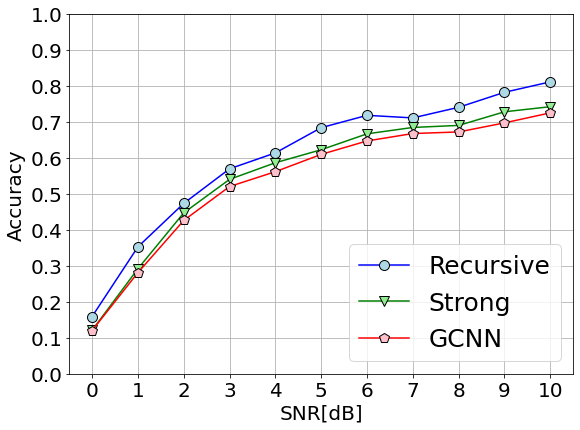}
		\caption{}
		\label{fig:exp2}
	\end{subfigure}
	\hfill	
	\begin{subfigure}[b]{0.22\textwidth}
		\centering
		\includegraphics[width=\textwidth]{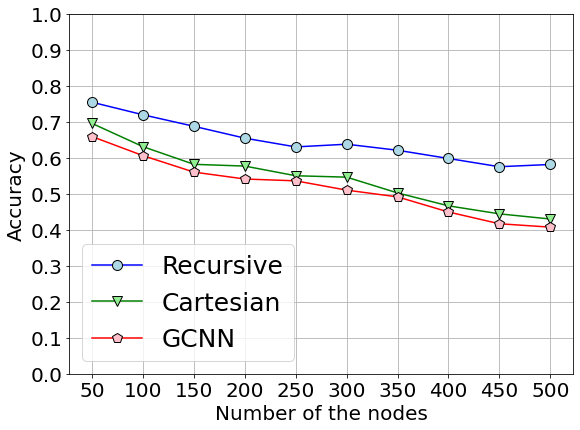}
		\caption{}
		\label{fig:exp3}
	\end{subfigure}
	\hfill
	\begin{subfigure}[b]{0.22\textwidth}
		\centering
		\includegraphics[width=\textwidth]{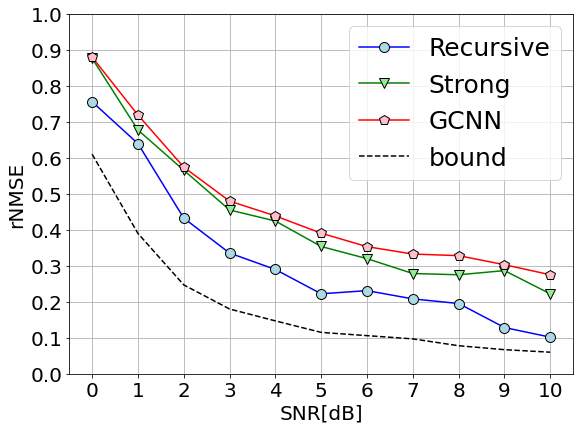}
		\caption{}
		\label{fig:exp4}
	\end{subfigure}
	\hfill
	\begin{subfigure}[b]{0.22\textwidth}
		\centering
		\includegraphics[width=\textwidth]{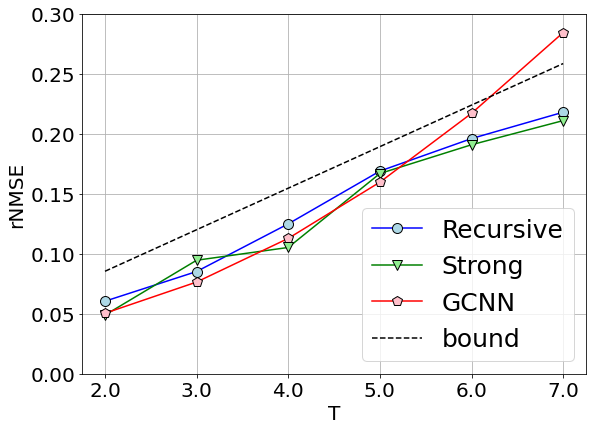}
		\caption{}
		\label{fig:exp5}
	\end{subfigure}\hfill
	\caption{
			 Stability results for different scenarios of the parametric GTCNN and the best alternative for a fixed product graph (Kronecker, Cartesian, Strong). We also consider the vanilla GCNN as a baseline.
			 (a) Results for different SNRs.
			 (b) Performance for different graph sizes in 5dB perturbation.
			 (c) Embedding difference in terms of rNMSE between the trained GTCNN on the nominal graph and the perturbed one.
			 (d) Embedding difference for different time series lengths in 5dB perturbation.
			}
	\label{fig:staball}
\end{figure*}

To investigate the stability of the GTCNN, we tested the trained models in the source localization experiment under perturbed graphs with different signal to noise ratios (SNR)
\begin{equation}
	\textnormal{SNR} = 10\log_{10} \frac{\|\bbS\|^2_F}{2\|\bbE\|_F^2}.
\end{equation}
The noise energy is doubled as it appears twice in the relative perturbation model [c.f.~\eqref{relperturb}].

Fig.~\ref{fig:exp2} shows the average classification accuracy for different amounts of perturbations.
The GTCNN performs decently even in noisy scenarios around $5dB$.
Fig.~\ref{fig:exp3} represents the performance for different number of nodes in the graph.
We trained the GTCNN over different graph sizes and evaluated the results using a perturbed graph to observe the effect of graph size on stability.
The accuracy reduces steadily with the graph size as Theorem~\ref{filterstanility} suggests.

To investigate GTCNN stability itself and its relation with bound~\eqref{stabilitybound}, we analyzed the output embeddings of a fixed GTCNN network between the nominal and the perturbed graph.
Fig.~\ref{fig:exp4} and~\ref{fig:exp5} depicts theoretical bound versus GTCNN empirical performance.
We can observe that the bound reflects the same behavior of the empirical results with respect to both noise energy and time series length.
In addition, from all results we see the GTCNN offers a more stable performance compared with the vanilla GCNN by using the temporal data as features. This in turn highlights our theoretical insights after Theorem~\ref{filterstanility}.

Finally, in Fig.~\ref{fig:ffrfr} we investigate the frequency response of the learned filters.	
The frequency responses vary smoothly over high graph frequencies for the trained filters while the temporal frequencies have more variations throughout of the spectrum.

\begin{figure*}
		\centering
	\begin{subfigure}[b]{0.45\textwidth}
		\centering
		\includegraphics[width=\textwidth]{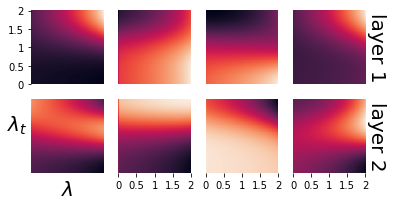}
		\caption{}
		\label{fig:ffr}
	\end{subfigure}
	\hfill	
	\begin{subfigure}[b]{0.45\textwidth}
		\centering
		\includegraphics[width=\textwidth]{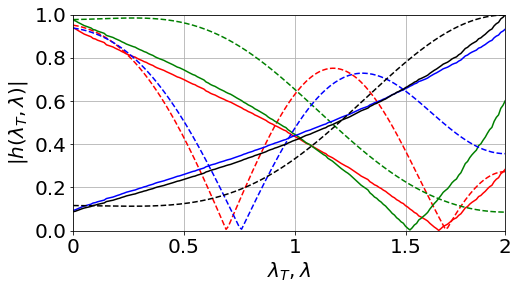}
		\caption{}
		\label{fig:ffr2}
	\end{subfigure}
	\hfill
	\caption{ (a)~Normalized frequency response of a trained recursive GTCNN with two layers and four filters per layer. The bright color represents 1 while the dark color stands for 0.
	(b)~Examples of the filter frequency response variation to normalized frequencies.
	continuous line belongs to graph frequencies $\lambda$ and dashed ones are related to temporal frequencies $\lambda_T$. }	
\label{fig:ffrfr}
\end{figure*}

\section{Conclusion}
\vfill
\label{sec_conc}
We introduced graph-time convolutional neural networks as a model to learn from spatiotemporal data.
The GTCNN uses product graph to convert dynamic data over network into static data over a larger graph.
Afterward, a shift-and-sum convolution mechanism conveys the information over the product graph to exploit spatiotemporal dependencies in the data.
The product graph itself can also be parametric to enable the model to learn temporal relationships directly from data, and also allows us to implement the GTCNN recursively and avoid storing and processing large graphs.
We proposed a spectral domain analysis for graph-time convolutional filters and showed they operate as point-wise multiplication between the filters frequency response and the time series graph-time Fourier transform. Such a spectral analysis allowed us to study the stability of the GTCNN to perturbations in the spatial support.
We showed that the GTCNN becomes linearly less stable as the time series length increases.
However, it is yet more stable than the case where we consider time as distinct features and approach the problem via GCNNs.
The numerical results also approved that GTCNNs performs better than a baseline GCNN model.
Moreover, GTCNN compares well on benchmark datasets to state-of-the-art graph-based models but it suffered to capture long term patterns in the data.
This is because the GTCNN works with large graphs and cannot maintain long input time series.
The presented recursive model overcomes the memory consumption problem for large graphs and long time series while computational complexity needs to be improved yet as a direction for future works. 

\vfill


\bibliographystyle{IEEEtran}
\bibliography{myIEEEabrv,GTCNN_lib}
%
%
%

\newpage

\appendices
\section{Proof of Proposition~\ref{prop1}}
\label{sec:gtfilter_proof}
The graph time convolutional filter in~\eqref{eq:graphtimefilter} operating over the parametric product graph is
\begin{equation}
	\bby_\diamond = \sum_{k=0}^K h_k \left(\sum_{i = 0}^1 \sum_{j = 0}^1 s_{ij}  \left(\mathbf{S}_{T}^{i} \otimes  \mathbf{S}^{j}  \right)\right)^k \bbx_\diamond.
	\label{eq:paramgtfilter}
\end{equation}
To expand equation~\eqref{eq:paramgtfilter}, we first investigate the powers of the parametric product graph
\begin{equation}
	\bbS_\diamond^k =
	(s_{00}\bbI_T\otimes\bbI_N+
	s_{01}\bbI_T\otimes\bbS+
	s_{10}\bbS_T\otimes\bbI_N+
	s_{11}\bbS_T\otimes\bbS_N)^k.
	\label{eq:powers} 
\end{equation}
Equation~\eqref{eq:powers} is expandable by multinomial theorem. The expanded version after applying mixed-product property of Kronecker product is
\begin{align}
	&\sum_{\mathclap{\sum k_{ij} = k}}c_\ccalI(\bbI_T\otimes\bbI_N)^{k_{00}}
	(\bbI_T\otimes\bbS)^{k_{01}}
	(\bbS_T\otimes\bbI_N)^{k_{10}}
	(\bbS_T\otimes\bbS_N)^{k_{11}}
	\nonumber \\
	&= \sum_{\mathclap{\sum k_{ij} = k}}c_\ccalI(\bbS_T^{k_{10}+k_{11}}\otimes\bbS^{k_{01}+k_{11}})
	\label{eq:powersexpanded}
\end{align}
where $\ccalI = \{k_{00},k_{01},k_{10},k_{11}\}$ is the multinomial term index and $c_\ccalI$ is the respective term coefficient.

Substituting the parametric graph powers~\eqref{eq:powersexpanded} into the parametric graph time convolutional filter~\eqref{eq:paramgtfilter} leads to
\begin{equation}
	\bby_\diamond = \sum_{k=0}^K h_k
	\sum_{k_{00}+k_{01}+k_{10}+k_{11} = k}c_\ccalI(\bbS_T^{k_{10}+k_{11}}\otimes\bbS^{k_{01}+k_{11}})
	\bbx_\diamond.
	\label{eq:doublesum}
\end{equation}
By a change of variables $l=k_{10}+k_{11}$ and $p=k_{01}+k_{11}$, we reduce~\eqref{eq:doublesum} into
\begin{equation}
	\bby_\diamond = \sum_{k=0}^K h_k
	\sum_{l+p = k}c_{lp}(\bbS_T^{l}\otimes\bbS^{p})
	\bbx_\diamond.
	\label{tmp10}
\end{equation}
By further changing the indexing and order of summation we can write~\eqref{tmp10} as
\begin{equation}
	\bby_\diamond = \sum_{l=0}^{K}\sum_{p=0}^{K-l} h_{lp}c_{lp}(\bbS_T^{l}\otimes\bbS^{p})
	\bbx_\diamond.
	\label{tmp11}
\end{equation}
Hence, the graph time convolutional filter in~\eqref{eq:paramgtfilter} is equivalent to~\eqref{tmp11}, which in turn is a specific case of~\eqref{eq:graphtimegeneralfilter} with $\bar{K} = K - l$, $\tilde{K} = K$, and $h_{kl} = h_{lp}c_{lp}$.


\section{Proof of Proposition~\ref{prop2}}
\label{sec:permequi_proof}
Let $\ccalG$ and its permuted version $\ccalG^\prime$ be represented by the GSOs $\bbS$ and $\bbS^\prime = \bbP^\top\bbS\bbP$, respectively, and let $\bbX^\prime = \bbP^\top\bbX$ be the spatial-graph permuted version of $\bbX$.
Considering the graph-time convolutional filter in~\eqref{eq:graphtimegeneralfilter}, the output for permuted signal and spatial graph is
\begin{equation}
	\bby_\diamond^\prime=
	\sum_{k=0}^{\bar{K}}
	\sum_{l=0}^{\tilde{K}}
	h_{kl} (\bbS_T^l \otimes (\bbP^\top\bbS\bbP)^k)\text{vec}(\bbP^\top\bbX)
	\label{eq:permeq_1}
\end{equation}
Invoking properties of permutation matrix $\bbP^k = \bbP$ and $\bbP^\top\bbP = \bbI_N$, the permuted spatial shift can be written as $(\bbP^\top\bbS\bbP)^k = \bbP^\top\bbS^k\bbP $.
Applying the mixed-product property for Kronecker product, we can rewrite~\eqref{eq:permeq_1} as
\begin{align}
	\bby_\diamond^\prime &=
	\sum_{k=0}^{\bar{K}}
	\sum_{l=0}^{\tilde{K}}
	h_{kl} (\bbI_T\bbS_T^l\bbI_T \otimes \bbP^\top\bbS^k\bbP)\text{vec}(\bbP^\top\bbX) \nonumber \\
	&= (\bbI_T\otimes\bbP^\top)\sum_{k=0}^{\bar{K}}
	\sum_{l=0}^{\tilde{K}}h_{kl} (\bbS_T^l \otimes \bbS^k)
	(\bbI_T\otimes\bbP)\text{vec}(\bbP^\top\bbX) \nonumber \\
	&=(\bbI_T\otimes\bbP^\top)\bbH(\bbS_T,\bbS)(\bbI_T\otimes\bbP)\text{vec}(\bbP^\top\bbX).
	\label{eq:permeq_2}
\end{align}
Now, we employ vector operator property $\text{vec}(\bbA\bbX\bbB)= (\bbB^\top \otimes \bbA)\text{vec}(\bbX)$ followed by the fact that permutation matrix is unitary $\bbP^\top\bbP = \bbI_N$ to write
\begin{equation}
	(\bbI_T\otimes\bbP)\text{vec}(\bbP^\top\bbX) = \text{vec}(\bbP\bbP^\top\bbX)
	=\text{vec}(\bbX).
	\label{tmp22}
\end{equation}
By replacing~\eqref{tmp22} into~\eqref{eq:permeq_2}, we have
\begin{align}
	\bby_\diamond^\prime
	&= (\bbI_T\otimes\bbP^\top)\bbH(\bbS_T,\bbS)\text{vec}(\bbX) \nonumber \\ 
	&=
	(\bbI_T\otimes\bbP^\top)\text{vec}(\bbY).
	\label{tmp23}
\end{align}
Expression~\eqref{tmp23} alongside with the vector operator property leads to
\begin{equation}
	\bby_\diamond^\prime=
	\text{vec}(\bbP^\top\bbY)
	\label{eq:permeq_3}
\end{equation}
Hence, $\bbY^\prime = \bbP^\top\bbY$ and the output is permuted similarly as the input.
Combining the results from~\eqref{eq:permeq_1} and~\eqref{eq:permeq_3} shows that the graph-time convolutional filter in~\eqref{eq:graphtimegeneralfilter} is permutation equivariant, i.e.,
\begin{equation}
		\text{vec}(\bbP^\top\text{vec}^{-1}(\bbH(\bbS_T,\bbS)\bbx_{\diamond})) = \bbH(\bbS_T,\bbP^\top\bbS\bbP)\text{vec}(\bbP^\top\bbX).
\end{equation}

Since GTCNNs are a composition of graph-time convolutional networks and point-wise nonlinearities, they also benefit from permutation equivariancy.
%

\section{Proof of Proposition~\ref{prop3}}	
\label{sec:spectral}
Considering the graph-time filter in~\eqref{eq:graphtimegeneralfilter},  the input-output relation is
\begin{equation}
	\bby_{\diamond} =
	\sum_{k=0}^{\bar{K}}
	\sum_{l=0}^{\tilde{K}}
	h_{kl} (\bbS_T^l \otimes \bbS^k)
	\bbx_{\diamond}.
	\label{eq:input-output} 
\end{equation}
We apply graph-time Fourier transform by multiplying both sides with
$(\bbV_T\otimes\bbV)^H$
from left
\begin{equation}
	\tilde{\bby}_\diamond =
	\sum_{k=0}^{\bar{K}}
	\sum_{l=0}^{\tilde{K}}
	h_{kl}
	(\bbV_T\otimes\bbV)^H
	(\bbS_T^l \otimes \bbS^k)
	\bbx_{\diamond},	
\end{equation}
where
$\tilde{\bby}_\diamond$
is the graph-time Fourier transform of output $\bby_{\diamond}$.
Leveraging the mixed-product property of Kronecker product leads to
\begin{equation}
	\tilde{\bby}_\diamond =
	\sum_{k=0}^{\bar{K}}
	\sum_{l=0}^{\tilde{K}}
	h_{kl}
	(\bbV_T^H\bbS_T^l \otimes \bbV^H\bbS^k)
	\bbx_{\diamond}.
\end{equation}
Using eigenvalue decomposition we have
\begin{equation}
	\tilde{\bby}_\diamond =
	\sum_{k=0}^{\bar{K}}
	\sum_{l=0}^{\tilde{K}}
	h_{kl}
	(\bbLambda_T^l\bbV_T^H \otimes \bbLambda^k\bbV^H)
	\bbx_{\diamond}.
\end{equation}
Then, the Kronecker product property can be applied again
\begin{equation}
	\tilde{\bby}_\diamond =
	\sum_{k=0}^{\bar{K}}
	\sum_{l=0}^{\tilde{K}}
	h_{kl}
	(\bbLambda_T^l \otimes \bbLambda^k)
	(\bbV_T\otimes\bbV)^H
	\bbx_{\diamond}.
\end{equation}
The term
$(\bbV_T\otimes\bbV)^H\bbx_{\diamond}$ shows the graph-time Fourier transform of input $\bbx_{\diamond}$.
Hence we can complete the proof as
\begin{equation}
	\tilde{\bby}_\diamond = 
	h(\bbLambda_T,\bbLambda)
	\tilde{\bbx}_\diamond,
\end{equation}
where
\begin{equation}
	h(\bbLambda_T,\bbLambda) =
	\sum_{k=0}^{\bar{K}}
	\sum_{l=0}^{\tilde{K}}
	h_{kl}
	(\bbLambda_T^l \otimes \bbLambda^k)	.
\end{equation}

\section{Proof of Theorem~\ref{filterstanility}}
\label{sec:stability_proof}
The proof follows closely as that for the GCNNs in~\cite{gamamain}, but the product graph requires a few changes. We will also use the following lemma. 
\begin{lemma}[c.f.~\cite{gamamain}]
	Let $\bbS=\bbV \bbLam \bbV^\herm$ and $\bbE = \bbU \bbM \bbU^\herm$ such that $\|\bbE\| \leq \eps$.
	Assume that $\bbE_V = \bbV \bbM \bbV^\herm$ is the projection of perturbation $\bbE$ over graph eigenspace, and $\bbE = \bbE_V+\bbE_U$.
	For any eigenvector $\bbv_i$ of $\bbS$ it holds that
	\begin{equation}
		\bbE \bbv_i = m_i \bbv_i + \bbE_U \bbv_i
		\end{equation}
	with $\|\bbE_U\|\leq \eps\delta$, where $\delta = (\|\bbU -\bbV\|^2+1)^2-1$ and $m_i$ is the $i$-th eigenvalue of $\bbM$.
	Recall that $\|\cdot\|$ represents the operator norm of a matrix.
	\label{mainlemma}
\end{lemma}

The proof is organized into two parts. First, we prove graph-time convolutional filter [cf.~\eqref{eq:graphtimegeneralfilter}] is stable and then we extend it to the GTCNN.

\smallskip
\noindent\textbf{Filter.}
Let $\bbx_\diamond$ be the input signal over product graph with finite energy, i.e., $\|\bbx_\diamond\|_2<\infty$.
The effect of perturbation on the output is
\begin{equation}
	\| \hat{\bby} - \bby \|_2 = \left\| [\bbH(\bbS_T,\bbS) - \bbH(\bbS_T,\hat{\bbS})]\bbx_\diamond\right\|_2.
	\label{outputdiff}
\end{equation}
The difference between convolutional filter operators can be further written as
\begin{equation}
	\bbH(\bbS_T,\hat{\bbS}) - \bbH(\bbS_T,\bbS) = \sum_{k=0}^{\bar{K}}\sum_{l=0}^{\tilde{K}} h_{kl}[\bbS_T^l \otimes (\hat{\bbS}^k-\bbS^k)].
	\label{filterdiff}
	\end{equation}
As in~\eqref{filterdiff}, we discussed the case for filters of orders infinity which allows us to link these discrete filters to their continuous spectral response.
The difference between spatial shift operators, $\bbS$ and $\hat{\bbS}$, can be expanded by first-order Taylor series approximation as
\begin{equation}
	\hat{\bbS}^k-\bbS^k = \sum_{r=0}^{k-1}(\bbS^r\bbE\bbS^{k-r}+\bbS^{r+1}\bbE\bbS^{k-r-1}) + \bbD
	\label{gsodiff}
	\end{equation}
where $\bbD = \ccalO(\eps^2)$ is negligible.

With this in place, we now move to the spectral domain. The graph-time Fourier decomposition for the signal $\bbx_\diamond$ over the product graph is
\begin{equation}
	\bbx_\diamond =  (\bbV_T \otimes \bbV)\tilde{\bbx}_{\diamond} = \sum_{t=1}^{T}\sum_{i=1}^{N} \tilde{\bbx}_{ti} (\bbv_{T,t}\otimes\bbv_i).
	\label{GFT}
\end{equation}
Substituting expansion~\eqref{GFT} of $\bbx_\diamond$ and~(\ref{gsodiff}) into~(\ref{outputdiff}), we have
\begin{align}
	\| \hat{\bby} - \bby \|&_2= \left\| \sum_{t=1}^{T}\sum_{i=1}^{N}\tilde{\bbx}_{ti}\sum_{k=0}^{\bar{K}}\sum_{l=0}^{\tilde{K}} h_{kl} \label{bigone}\right.\\
	&\left. [\bbS_T^l \otimes \sum_{r=0}^{k-1}(\bbS^r\bbE\bbS^{k-r}+\bbS^{r+1}\bbE\bbS^{k-r-1})] (\bbv_{Tt}\otimes\bbv_i) \right\|_2 \nonumber
\end{align}

Now, we expand the inner summation in (\ref{bigone}) to simplify it. By using the mixed-product property of Kronecker product
$(\bbA \otimes \bbB)(\bbC \otimes \bbD) = (\bbA\bbC)\otimes(\bbB\bbD)$, we have
\begin{align}
	 &[\bbS_T^l \otimes \sum_{r=0}^{k-1}(\bbS^r\bbE\bbS^{k-r}+\bbS^{r+1}\bbE\bbS^{k-r-1})] (\bbv_{Tt}\otimes\bbv_i) \nonumber \\
	&=\bbS_T^l\bbv_{Tt} \otimes \sum_{r=0}^{k-1}(\bbS^r\bbE\bbS^{k-r}+\bbS^{r+1}\bbE\bbS^{k-r-1})\bbv_i.
	\label{mixedproduct}
\end{align}
With the eigenvalue definition $\bbS\bbv_i = \lambda_i\bbv_i$, we can reduce (\ref{mixedproduct}) into
\begin{equation}
	\lambda_{T,t}^l\bbv_{T,t} \otimes \sum_{r=0}^{k-1}(\lambda_i^{k-r}\bbS^r+\lambda_i^{k-r-1}\bbS^{r+1})\bbE\bbv_i.
	\label{evreplaced}
\end{equation}

By applying Lemma~\ref{mainlemma} on~(\ref{evreplaced}), we have
\begin{equation}
	\lambda_{Tt}^l\bbv_{Tt} \otimes\sum_{r=0}^{k-1}2m_i\lambda_i^k\bbv_i +(\lambda_i^{k-r}\bbS^r+\lambda_i^{k-r-1}\bbS^{r+1})\bbE_U\bbv_i
	\label{lemmareplaced}
\end{equation}
where recall $m_i$ is the $i$-th eigenvector of $\bbM$ and $\bbE_U = \bbE - \bbV\bbM\bbV^\herm$.
Then, by replacing (\ref{lemmareplaced}) in (\ref{bigone}) we get
\begin{align}
	&\| \hat{\bby} - \bby \|_2= \left\|\sum_{t=1}^{T}\sum_{i=1}^{N}\tilde{\bbx}_{ti}\sum_{k=0}^{\bar{K}}\sum_{l=0}^{\tilde{K}} h_{kl} \label{newbigone} \right.\\
	&\left.\lambda_{Tt}^l\bbv_{Tt} \otimes\sum_{r=0}^{k-1}2m_i\lambda_i^k\bbv_i +(\lambda_i^{k-r}\bbS^r+\lambda_i^{k-r-1}\bbS^{r+1})\bbE_U\bbv_i\right\|_2. \nonumber
\end{align}
The first term of inner summation in~\eqref{newbigone}
\begin{equation}
	\bbt_1 = 2\sum_{t=1}^{T}\sum_{i=1}^{N}\tilde{\bbx}_{ti}m_i\sum_{k=0}^{\bar{K}}\sum_{l=0}^{\tilde{K}} k h_{kl}\lambda_{Tt}^l\lambda_i^k (\bbv_{Tt}\otimes\bbv_i)
\end{equation}
can be simplified based on partial derivation definition
\begin{equation}
	\bbt_1 = 2\sum_{t=1}^{T}\sum_{i=1}^{N}\tilde{\bbx}_{ti}m_i \lambda_i\frac{\partial h(\lambda_{Tt},\lambda_i)}{\partial \lambda_i} (\bbv_{Tt}\otimes\bbv_i).
	\end{equation}
Using inequalities~\eqref{eq:lipsch} and $\|\bbE\|\leq \eps$, we can bound the norm of this term by
\begin{equation}
	\|\bbt_1\|_2 \leq 2C\eps\|\bbx_\diamond\|_2 
	\label{1sttermbound}
\end{equation}
where $C$ is the integral Lipschitz constant of graph-time convolutional filters. 

For the second term of the inner summation in~\eqref{newbigone}, we use the mixed-product property again to write
\begin{align}
	\bbt_2 = \sum_{t=1}^{T}\sum_{i=1}^{N}\tilde{\bbx}_{ti}&(\bbI_T\otimes \bbV) 
	(\bbI_T \otimes \diag(\bbg_i)) \label{secondterm} \\
	&(\bbI_T\otimes \bbV^H)(\bbI_T\otimes \bbE_U) (\bbv_{Tt}\otimes\bbv_i), \nonumber
	\end{align}
where 
\begin{align}
	[\bbg_i]_j &= \sum_{k=0}^{\bar{K}}\sum_{l=0}^{\tilde{K}} k h_{kl}\lambda_{Tt}^l\sum_{r=0}^{k-1}\lambda_i^{k-r}\lambda_j^r+\lambda_i^{k-r-1}\lambda_j^{r+1} \label{gdef} \\
	&=\left\{
	\begin{matrix*}[l]
		&\lambda_i\frac{\partial h(\lambda_{Tt},\lambda_i)}{\partial \lambda_i} & \text{if}\quad i=j \\
		&\frac{\lambda_i+\lambda_j}{\lambda_i-\lambda_j}(h(\lambda_{Tt},\lambda_i)-h(\lambda_{Tt},\lambda_j))&\text{if}\quad i=j
		\end{matrix*}
	\right.. \nonumber
	\end{align}
From the integral Lipschitz filter (Def.~\ref{def:intlip}), each entry of $\bbg_i$ is bounded by $2C$, hence the operator norm of $\|\diag(\bbg_i)\|\leq 2C$.
Due to the Lemma~1 $\|\bbE_U\|\leq \eps\delta$, and then $\| \bbI_T \otimes \bbE_U \|$ is bounded by $\eps\delta\sqrt{T}$.
Finally, the summation over all frequency components multiplies the bound by $\sqrt{NT}$, so
\begin{equation}
	\|\bbt_2\|_2 \leq 2C\eps\delta T \sqrt{N} \|\bbx_{\diamond}\|_2.
	\label{2ndtermbound}
\end{equation}

Replacing~\eqref{1sttermbound} and~\eqref{2ndtermbound} bounds into~\eqref{newbigone} followed by triangle inequality leads to the upper bound on the filter output
\begin{equation}
	\|\hat{\bby} - \bby\|_2 \leq 2C\eps(1+\delta T \sqrt{N}) \|\bbx_\diamond\|_2 = \Delta \eps \|\bbx_\diamond\|_2
	\label{eq:filterbound}
\end{equation}
where $\Delta = 2C(1+\delta T \sqrt{N})$.

\smallskip
\noindent\textbf{GTCNN.}
We can write the output difference of a GTCNN with $L$ layers over the nominal graph $\bbx_{\diamond,L}$ and perturbed graph $\hat{\bbx}_{\diamond,L}$ as
\begin{align}
	&\|\bbx_{\diamond,L} - \hat{\bbx}_{\diamond,L}  \|_2 \leq \label{gtcnn_output_diff} \\
	&\left\|
	\sigma\left(\sum_{g=1}^F \bbH_L^{fg}\bbx_{\diamond,L-1}^g\right) - 
	\sigma\left(\sum_{g=1}^F \hat{\bbH}_L^{fg}\hat{\bbx}_{\diamond,L-1}^g\right)
	\right\|_2,
	\nonumber
\end{align}
where, for the sake of notation simplicity, $\bbH_L^{fg}$ and $\hat{\bbH}_L^{fg}$ are standing for $\bbH_L^{fg}(\bbS_T,\bbS)$ and $\bbH_L^{fg}(\bbS_T,\hat{\bbS})$, respectively.

Applying the 1-Lipschitz continuous property of activation function $|\sigma(a)-\sigma(b)|\leq |a-b|$, followed by triangular inequality leads into 
\begin{equation}
	\|\bbx_{\diamond,L} - \hat{\bbx}_{\diamond,L}  \|_2 \leq
	\sum_{g=1}^F \|\bbH_L^{fg}\bbx_{\diamond,L-1}^g - \hat{\bbH}_L^{fg}\hat{\bbx}_{\diamond,L-1}^g \|_2.
	\label{gtcnn_output_diff_simple}
\end{equation}
Adding and subtracting $\hat{\bbH}_L^{fg}\bbx_{\diamond,L-1}^g$ and leveraging again the triangular inequality we can upper bound~\eqref{gtcnn_output_diff_simple} as 
\begin{align}
	\|\bbx_{\diamond,L} - \hat{\bbx}_{\diamond,L}  \|_2 &\leq 
	\|(\bbH_L^{fg}-\hat{\bbH}_L^{fg})\bbx_{\diamond,L-1}^g\|_2 \nonumber\\
	&+\|\hat{\bbH}_L^{fg}(\bbx_{\diamond,L}^g-\hat{\bbx}_{\diamond,L-1}^g)\|_2.
	\label{gtcnn_addsub}
\end{align}
The first term is bounded by $\Delta\eps$ [cf.~\eqref{eq:filterbound}] being it related to the filter stability.
The second term is bounded by spectral response of graph-time convolutional filters $\|\hat{\bbH}_L^{fg}\|\leq 1$ based on the theorem assumption [cf. Def~\ref{def:norm}].
Replacing these bounds into~\eqref{gtcnn_addsub} leads to
\begin{equation}
	\|\bbx_{\diamond,L} - \hat{\bbx}_{\diamond,L}  \|_2 \leq
	\sum_{g=1}^F \Delta \eps \|\bbx_{\diamond,L-1}^g\|_2 +  \|\bbx_{\diamond,L-1}^g-\hat{\bbx}_{\diamond,L-1}^g\|_2.
	\label{gtcnn_bound_raw}
\end{equation}
This equation indicates a recursion where the bound in the last layer depends on the bound in the previous layer.
We can write the output norm of each layer as
\begin{equation}
	\|\bbx_{\diamond,\ell}^f\|_2 \leq \sum_{g=1}^F \|\bbx_{\diamond,\ell-1}^g \|_2
	\label{gtcnn_bound_layers}
\end{equation}
Substituting~\eqref{gtcnn_bound_raw} and~\eqref{gtcnn_bound_layers} and solving the recursion bounds as in~[(87)-(88), 13] we can bound the output difference norm as
\begin{equation}
	\|\Phi(\bbx_{\diamond};\bbS_T,\bbS)-\Phi(\bbx_{\diamond};\bbS_T,\hat{\bbS})\|_2 \leq
	L F^{L-1}\Delta\eps \|\bbx_\diamond\|_2.
\end{equation}







%





\ifCLASSOPTIONcaptionsoff
  \newpage
\fi

\end{document}